%% file: main.tex
\documentclass{article}

\usepackage[
  margin=1in,               
  includefoot, heightrounded,
  marginparwidth=3.5cm,     
  marginparsep=2mm          
]{geometry}
\usepackage{times} 

\usepackage[parfill]{parskip}

\input{math_commands.tex}
\usepackage[utf8]{inputenc}  
\usepackage[T1]{fontenc}    
\usepackage{url}           
\usepackage{booktabs}     
\usepackage{amsfonts,bm}     
\usepackage{nicefrac}      
\usepackage{microtype}       
\usepackage{tikz}
\usepackage{svg}
\usetikzlibrary{fit,positioning,calc,shapes,arrows.meta}
\usepackage{subcaption}
\usepackage{soul}
\usepackage{csquotes}

\usepackage{natbib,hyperref}

\usepackage{color,xcolor}
\definecolor{darkblue}{rgb}{0, 0, 0.5}
\definecolor{darkgreen}{rgb}{0, 0.6, 0}
\definecolor{darkred}{rgb}{0.6, 0, 0}
\hypersetup{colorlinks=true, citecolor=darkblue, linkcolor=darkblue, urlcolor=darkblue}

\usepackage[colorinlistoftodos]{todonotes} 

\usepackage{amsmath}
\usepackage{amssymb}
\usepackage{mathtools} 
\usepackage{amsthm}
\usepackage{aliascnt}

\usepackage{amssymb}
\usepackage{xcolor}
\usepackage{wrapfig}

\usepackage{enumitem}

\setlist{nosep,leftmargin=0.3in} 
\usepackage{multirow}
\usepackage[noend]{algpseudocode}
\usepackage{algorithm}

\usepackage{graphicx}  
\usepackage{xurl}

\usepackage[capitalize,noabbrev,nameinlink]{cleveref}

\usepackage{mdframed}
\mdfsetup{
linecolor=black!40,
linewidth=0.5pt,
innerleftmargin = 2pt,
innertopmargin = 4pt,
innerrightmargin= 2pt,
innerbottommargin = 4pt,
}

\usepackage{listings}
\definecolor{codegreen}{rgb}{0,0.6,0}
\definecolor{codegray}{rgb}{0.5,0.5,0.5}
\definecolor{codepurple}{rgb}{0.58,0,0.82}
\definecolor{backcolour}{rgb}{0.95,0.95,0.92}
\lstdefinestyle{mystyle}{
    backgroundcolor=\color{backcolour},
    commentstyle=\color{codegreen},
    keywordstyle=\color{blue},
    numberstyle=\tiny\color{codegray},
    stringstyle=\color{codepurple},
    basicstyle=\ttfamily\small,
    breakatwhitespace=false,
    breaklines=true,
    captionpos=b,
    keepspaces=true,
    numbers=left,
    numbersep=5pt,
    showspaces=false,
    showstringspaces=false,
    showtabs=false,
    tabsize=2
}
\definecolor{lightgray}{gray}{0.9}
\newcommand{\inlinecode}[2]{\colorbox{lightgray}{\lstinline[language=#1]$#2$}}

\usepackage{tcolorbox}
\creflabelformat{equation}{#2\textup{#1}#3}

 \newtcolorbox[crefname={note}{notes}]{notebox}[1][]{%
    colback=blue!5!white,
    colframe=blue!10!white,
    sharp corners,
    before upper={{\bfseries Note~\thetcbcounter}.\ },
    #1
}
\newtcolorbox[auto counter, number within=section,crefname={insight}{insights}]{insightbox}[1][]{%
    colback=blue!5!white,
    colframe=blue!10!white,
    sharp corners,
    before upper={{\bfseries Insight~\thetcbcounter}.\ },
    #1
}

\newtcolorbox[auto counter, number within=section,crefname={insight}{insights}]{recommendationbox}[1][]{%
    colback=green!5!white,
    colframe=green!15!white,
    sharp corners,
    before upper={{\bfseries Recommendation~\thetcbcounter}.\ },
    #1
}

\theoremstyle{plain}
\newtheorem{theorem}{Theorem}[section]
\newtheorem{proposition}[theorem]{Proposition}

\theoremstyle{definition}
\newaliascnt{definition}{theorem}
\newtheorem{definition}[definition]{Definition}
\aliascntresetthe{definition}
\crefname{definition}{Definition}{Definitions}

\theoremstyle{remark}

\newlength\savewidth
\usepackage{xcolor}
\usepackage{tabulary}
\newcolumntype{x}[1]{>{\centering\arraybackslash}p{#1pt}}

\title{Next-Latent Prediction Transformers Learn Compact\\World Models}

\usepackage{neurips_titlepage}

\author{%
    Jayden Teoh\footnotemark[1]
  \qquad Manan Tomar
  \qquad Kwangjun Ahn
  \qquad Edward S.~Hu
  \qquad Tim Pearce\\[0.5em]
  \bfseries
    Pratyusha Sharma
  \qquad Akshay Krishnamurthy
  \qquad Riashat Islam
  \qquad Alex Lamb\\[0.5em]
  \bfseries John Langford \\[1.0em]
  Microsoft Research
}

\begin{document}

\maketitle 
{\renewcommand{\thefootnote}{\fnsymbol{footnote}}%
 \footnotetext[1]{Correspondence to: jayden\_t[at]mit[dot]edu}
}

\vspace{-1em}
\begin{abstract}
Transformers replace recurrence with a memory that grows with sequence length and self-attention that enables ad-hoc lookups over past tokens. Consequently, they lack an inherent incentive to compress history into compact latent states with consistent transition rules. This often leads to learning solutions that generalize poorly. 
We introduce \textbf{Next-Latent Prediction (NextLat)}, which extends standard next-token training with self-supervised predictions in the latent space. Specifically, NextLat trains a transformer to learn latent representations that are predictive of the next latent state given the next token. Theoretically, we show that these latents converge towards \emph{belief states}, compressed information about the history necessary to predict the future. This simple auxiliary objective injects a recurrent inductive bias into transformers while leaving their architecture, parallel training efficiency, and inference unchanged. NextLat effectively encourages transformers to form compact internal world models with coherent belief states and transition dynamics---crucial properties not guaranteed by standard next-token prediction alone. Empirically, across benchmarks in world modeling, reasoning, planning, and language modeling, NextLat demonstrates significant gains over standard next-token prediction and other baselines in downstream accuracy, representation compression, and lookahead planning. Furthermore, NextLat enables \emph{variable-length self-speculative decoding}, accelerating inference by up to $3.3\times$ in language modeling. NextLat offers a simple yet effective paradigm for learning compact, predictive representations in transformers that generalize better. Our code is available at \url{https://github.com/JaydenTeoh/NextLat}. 
\end{abstract}

\begin{figure}[H]
  \centering
  \includegraphics[width=1.0\linewidth]{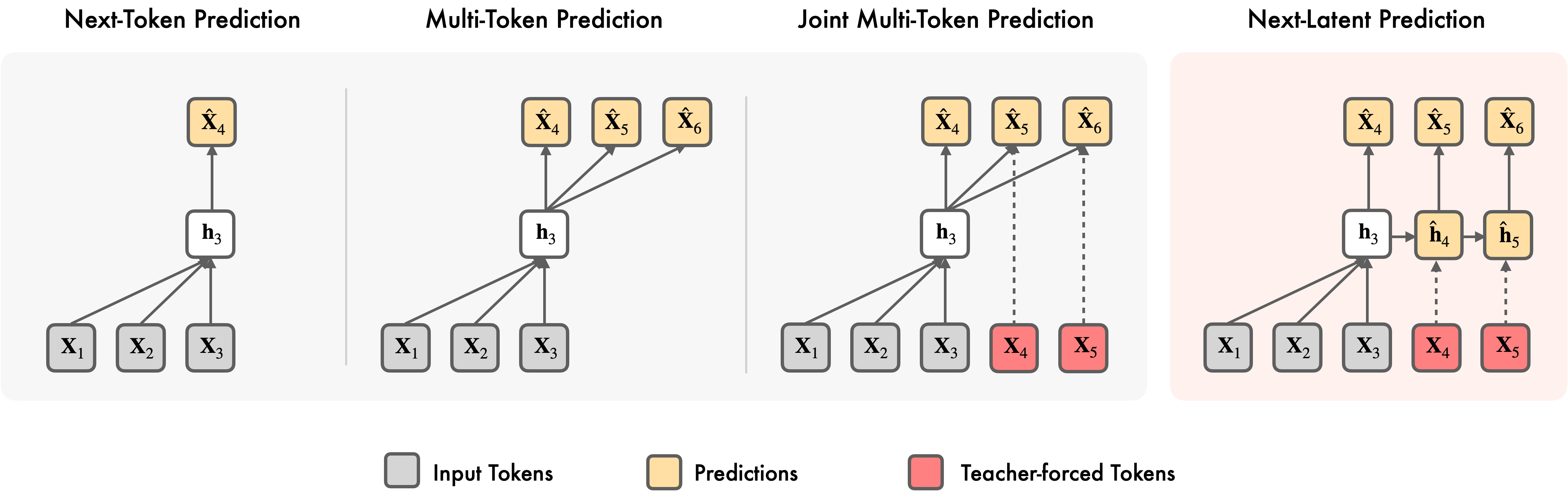}
  \caption{\textbf{Illustration of different predictive mechanisms.} Other methods supervise only in token space, leaving intermediate latent representations implicit. In contrast, NextLat explicitly trains the model to predict the next latent state ($h_{t+1}$) given the current latent state ($h_t$) and next token ($X_{t+1}$). Token supervision is then applied to the latent prediction ($\hat{h}_{t+1}$). Therefore, accurate multi-token predictions emerge from faithful latent dynamics modeling.}
  \label{fig:models}
\end{figure}

\input{content/intro}
\input{content/background}

\input{content/next_hidden_state}

\input{content/experiments}

\input{content/discussions_and_comparisons}
\input{content/conclusion}

\bibliography{ref} 
\bibliographystyle{plainnat}

\clearpage
\appendix

\input{content/appendix}

\end{document}

%% file: math_commands.tex
\newcommand{\stopgrad}{\textbf{sg}}

\newcommand{\fetchhead}{\mathrm{Fetch}}
\newcommand{\hidden}{\mathbf{h}}

%% file: content/intro.tex
\section{Introduction}

Ptolemy's geocentric model was able to accurately predict observations of the solar system from Earth's viewpoint, yet it was structurally convoluted: at times implying that the Moon came twice as close to Earth as at other times. It was later supplanted by Copernicus’s simpler, more compact heliocentric model, which provided accurate predictions that generalized beyond Earth’s perspective. In learning theory, it is well known that simpler explanations of training observations tend to generalize better \citep{blumer1986classifying, langford2005tutorial}.

Modern transformers \citep{transformers2017} stand in contrast to this principle.  By replacing recurrence with a memory that scales with sequence length and self-attention that enables flexible lookups over past tokens, they achieve exceptional parallelization and predictive power. Yet, this very capability removes any inherent pressure to compress history into compact latent representations with consistent update rules. As a result, transformers often learn complex, task-specific shortcuts that fit the training data well but generalize poorly \citep{anil2022lengthgeneralization,dziri2023faith,liu2023shortcuts,wu2024reasoning}. How can we encourage transformers to form simpler, more principled explanations that avoid such shortcuts? A natural approach is to reinstate a key property of recurrent models: the ability to learn \emph{compact world models} that channel future prediction through compressed representations of the past. We will show that this inductive bias can be reintroduced while also retaining the parallel training efficiency of transformers.

In this paper, we introduce \textbf{Next-Latent Prediction (NextLat)}, which extends the standard next-token prediction objective with self-supervised predictions in latent space. NextLat jointly trains a transformer and a latent dynamics model: the transformer learns to encode past tokens into compact latent summaries such that the dynamics model can predict the transformer’s next latent state given only the current latent state and the next token (i.e., the ``action''). This objective encourages the transformer to form a compact internal world model with coherent recurrent-like dynamics, while avoiding the sequential processing overhead of recurrent architectures. Importantly, NextLat leaves the transformer’s architecture, inference procedure, and parallel training efficiency unchanged, adding only a lightweight auxiliary loss on latent representations during training. By augmenting sparse one-hot token targets with dense latent-state supervision, NextLat provides richer learning signals and improves data efficiency compared to next-token prediction and other token-level supervision methods. Our approach is inspired by the \emph{self-predictive learning} paradigm in reinforcement learning (RL), a family of algorithms that learn representations by minimizing the prediction error of their own future latent states \citep{tang2023understanding,ni2024bridging}. 

Beyond representation learning, latent dynamics also provide an inference advantage: future states can be recursively predicted directly in latent space without invoking the main transformer. This recursive multi-step lookahead using the lightweight dynamics model enables \emph{variable-length self-speculative decoding}: the model can speculate a flexible number of future tokens and accelerate inference without requiring separate multi-token prediction heads.

The core contributions of this paper are threefold.
First, we establish a theoretical foundation showing that NextLat shapes transformer representations into \emph{belief states}---compact summaries of past information sufficient for predicting future observations. Such representations are important for planning and generalization, yet are not guaranteed to emerge from next-token prediction alone. 
Second, we present a practical implementation of NextLat that preserves the transformer’s architecture, inference procedure, and parallel training efficiency. 
Finally, we empirically demonstrate NextLat’s effectiveness across diverse domains spanning world modeling, reasoning, planning, and language modeling. Our results show that NextLat improves representation compactness, lookahead planning, and downstream accuracy over standard next-token prediction and other baselines. In language modeling, NextLat achieves up to $3.3\times$ faster inference through variable-length self-speculative decoding. Together, these results position NextLat as an efficient framework for learning compact, predictive, and generalizable representations in transformers.


%% file: content/background.tex
\section{Related Work}
We are motivated by a long line of prior works in representation learning for prediction and control. Of close interest are self-supervised learning methods, belief states for decision making, and world models.

\paragraph{Self-Supervised Learning.} Self-supervised learning (SSL) is a framework for learning from unlabeled data, where a model generates its own supervisory signals from the structure of raw inputs. Across modalities such as vision, audio, and time series, SSL has proven highly effective for pretraining useful features, enabling downstream transfer that rivals, or even surpasses, models trained on labeled data \citep{liu2022audioselfsupervisedlearningsurvey,balestriero2023cookbookselfsupervisedlearning,zhang2024selfsupervisedlearningtimeseries}.
There are several approaches to SSL. 
Our method falls under self-predictive representation learning, which jointly learns latent representations and a transition function that models how these representations evolve over a sequence. Self-prediction has driven state-of-the-art advances in RL \citep{gelada2019deepmdp,zhang2020learning,ye2021mastering,schwarzer2021dataefficientreinforcementlearningselfpredictive,hansen2022temporaldifferencelearningmodel}. \citet{herrmann2025measuring} train transformers to predict their own future hidden states through a variational bottleneck, using the resulting prediction error to measure the ``interestingness'' of in-context computation. However, their predictor operates at an intermediate layer and conditions on the full history of past latents, and therefore does not induce belief states. NextLat instead learns a Markovian, token-conditioned transition over the final hidden states, yielding compact belief-state representations and downstream gains. LLM-JEPA \citep{huang2025llm} minimizes distances between embeddings of paired text--code data, but relies on manually curated pairs and thus does not generalize to raw text.  In contrast, NextLat is fully self-supervised and requires no paired data, making it applicable to transformers across arbitrary sequence modeling domains.

\paragraph{Belief States.} In both sequence modeling and RL, models must reason over long histories of observations. 
To mitigate this curse of dimensionality, prior work focuses on compressing history into latent representations that capture all information necessary for future prediction. In RL, this latent summary is formalized by \citet{kaelbling1998planning} as a belief state, defined as: \emph{``a sufficient statistic for the past history \ldots\ no additional data about its past actions or observations would supply any further information about the current state of the world''}. In stochastic control, the same notion of sufficient statistics appears as ``information state'' \citep{striebel1965sufficient}. The idea of sufficient statistics is also key to learning state abstractions in RL \citep{li2006towards}. While recurrent neural networks naturally enforce such compression, transformers have no such constraint---their internal state, or memory, grows linearly with sequence length. Recently, Belief State Transformers (BST; \citep{hu2025the}) extended the notion of belief states to transformers, and demonstrated benefits in planning tasks. 
Compared to BST, NextLat learns belief states without requiring a separate transformer and is much more computationally efficient. We compare these methods further in \cref{section:discussions}.

\paragraph{World Models.} 
Loosely, a world model is an internal predictive model of how the world works, with varying interpretations across cognitive science \citep{craik1967nature,johnson1983mental}, neuroscience \citep{miall1996forward, friston2010free}, control theory \citep{francis1976internal,conant1970every} and reinforcement learning \citep{sutton1991dyna,schmidhuber1990making,ha2018world}. 
Whether transformer language models implicitly learn world models remains a debate; some studies report emergent world understanding \citep{patel2022mapping,li2023emergent,gurnee2024language}, while others find incoherent world structure \citep{vafa2024evaluating,vafa2025has}. However, successful world modeling approaches such as MuZero \citep{schrittwieser2020mastering} for achieving superhuman performance in video and board games, 
Dreamer \citep{hafner2019dreamerv1,hafner2021dreamerv2,hafner2025dreamerv3} for model-based RL, and Genie \citep{bruce2024geniev1} for interactive video generation share a common principle: they learn a latent dynamics model that takes a latent state (an encoding of past observations), an action, and predicts the next latent state. Yet, learning such latent dynamics for transformer-based language modeling remains underexplored. NextLat addresses this gap by explicitly learning a latent dynamics model that governs how the transformer's latent states evolve given new tokens (i.e., ``actions''), enabling the transformer to learn compact latent abstractions of the world with consistent dynamics.

\paragraph{Beyond Next-Token Prediction.} In the domain of language modeling, a growing body of work has highlighted the myopic nature of the next-token prediction objective, which limits the model's capability in downstream tasks such as planning and reasoning \citep{bachmann24a,nagarajan2025roll}. Recent works have also found improvements from richer supervision signals that predict further into the future \citep{gloeckle2024,liu2024deepseek,hu2025the,ahn2025jtp}. However, these approaches operate predominantly in the token space. NextLat takes a different approach: it shifts prediction into the latent space, enforcing coherent dynamics over the model’s latent representations rather than its token outputs. As we discuss later in \cref{subsection:why_nextlat}, this latent-space supervision provides richer gradient signals than token-level supervision.

\paragraph{Speculative Decoding.} Speculative decoding \citep{leviathan2022fast, chen2023acceleratinglargelanguagemodel} accelerates inference by using a lightweight draft model to propose multiple tokens, which are then verified in parallel by a target model.
Closely related, \citet{li2024eagle1, li2024eagle2} train a draft model on top of a frozen language model to predict its high-level latent features. Unlike our approach, this is done post-hoc, whereas we study latent prediction as a pretraining objective to directly shape representations. Moreover, their draft model relies on full attention over past features and therefore does not learn belief states, i.e., compact summaries of past tokens.
Self-speculative decoding removes the need for a separate draft model by having a single model act as both draft and verifier. A common implementation uses the multi-token prediction (MTP) objective, where the model is trained to predict multiple future tokens that can be verified in parallel \citep{gloeckle2024}. However, MTP operates in token space and is therefore typically constrained to a fixed speculative horizon determined during training. In contrast, NextLat learns a latent dynamics model that can be recursively composed in latent space, enabling variable-length self-speculative decoding and faster inference.

%% file: content/next_hidden_state.tex
\section{Methodology: Next-Latent Prediction}
\label{sec:NH} 

In this section, we introduce a simple, yet powerful, method for learning belief states in transformers via next-latent prediction (or more specifically, via \emph{next-hidden state prediction}\footnote{In the sequence modeling literature, intermediate latent representations are often referred to as ``hidden states''. To disambiguate, we use the term ``latent state'' to broadly refer to learned representations within the transformer's residual stream, and ``hidden state'' to refer to a subset of this representation---specifically, the final layer’s output at each time step (i.e., the pre-logit activations).}). We begin by defining belief states in sequence modeling.
\begin{definition}[Belief states in sequence modeling] \label{def:belief_states}
Let $X_{1:T}$ denote a token sequence $X_1, \dots, X_T$. A random variable $\mathbf{b}_t = g(X_{1:t})$ is a \emph{belief state} for the history $X_{1:t}$ if, for every bounded measurable function $f$ of the future,
\begin{align*}
\mathbb{E}[f(X_{t+1:T}) \mid \mathbf{b}_t] = \mathbb{E}[f(X_{t+1:T}) \mid X_{1:t}] \quad \text{a.s.}
\end{align*}
\end{definition}
Equivalently, $\mathbf{b}_t$ is a \emph{sufficient statistic}~\citep{striebel1965sufficient} of the history $X_{1:t}$ for predicting the future tokens, i.e., from which we can sample from the distribution $\mathbb{P}(X_{t+1:T}\mid X_{1:t})$. Next, we describe how next-latent prediction enables transformers to learn belief states and improves data efficiency.

\subsection{Why Next-Latent Prediction?}
\label{subsection:why_nextlat}

Here we analyze an idealized next-latent prediction transformer which successfully optimizes both next-token prediction and next-latent prediction with respect to an underlying data distribution.
\begin{theorem}
\label{theorem:nhs_belief_states}
Consider the joint learning of three components:
\begin{enumerate}
    \item a transformer with parameters $\theta$ that produces hidden states $\hidden_t$ at each time step $t$,
    \item an output head $p_\theta$ modeling the next-token distribution, and
    \item a latent dynamics model $p_\psi$ modeling the transition dynamics of the transformer’s hidden states.
\end{enumerate}

If NextLat successfully optimizes the following objectives:
\begin{align}
\textbf{(Next-Token Consistency):} \quad &p_\theta(X_{t+1}\mid \hidden_t) =\mathbb{P}(X_{t+1}\mid X_{1:t}),\label{eq:emission_correctness} \\
\textbf{(Transition Consistency):} \quad &p_\psi(\hidden_{t+1} \mid \hidden_t, X_{t+1}) = \mathbb{P}(\hidden_{t+1} \mid  X_{1:t+1}), \label{eq:transition_correctness}
\end{align}
then $\hidden_t$ must be a belief state for the sequence $X_{1:t}$. Note that the right-hand side of \cref{eq:transition_correctness} is the transition law induced by the transformer's weights\footnote{We adopt a probabilistic formulation to retain generality with respect to stochastic transformer models, e.g., \citet{fleuret2025freetransformer}.}. 
\end{theorem}
\textit{Proof Sketch.} A formal proof by backward induction is provided in \cref{pf:nhs_belief_states}. Intuitively, optimizing for next-token (\cref{eq:emission_correctness}) and transition (\cref{eq:transition_correctness}) consistency ensures existence of measurable maps, i.e., $p_\theta$ and $p_\psi$, that allow recursive decoding of future tokens from $\hidden_t$:
    \begin{align*}
        \hidden_t 
        &\xrightarrow[\text{decode token}]{p_\theta} X_{t+1} 
        \xrightarrow[\text{update state}]{p_\psi} \hidden_{t+1} 
        \xrightarrow[\text{decode token}]{p_\theta} X_{t+2} 
        \xrightarrow[\text{update state}]{p_\psi} \hidden_{t+2} 
        \;\cdots\;
        \xrightarrow[]{p_\theta} X_T.
    \end{align*}
For these maps to exist, and be learned, $\hidden_t$ must jointly optimize toward a belief state---a sufficient statistic for the history to predict the future. 

\textit{Remark.} Optimizing only next-token consistency (i.e., \cref{eq:emission_correctness}) in standard autoregressive transformers does not guarantee that $\hidden_t$ forms a belief state (see Theorem 3 in \citet{hu2025the}). Intuitively, self-attention enables ad-hoc lookup of past tokens, so there is no pressure to compress all necessary information about the past into compact latent summaries at every time step.

\paragraph{Better Data Efficiency.}
A distinctive feature of NextLat is the richness of its learning signal. In prior methods that supervise in token space (see \cref{fig:models}), the learning signal is only anchored to the next token, or multiple tokens. NextLat additionally supervises in latent space. Specifically, the model is trained to predict its own next hidden state $\hidden_{t+1}$, which parameterizes the full predictive distribution over $X_{t+2}$. This shifts supervision from individual one-hot token labels to distribution-level alignment. Moreover, because the latent dynamics compose recursively---each latent is trained to predict the next---$\hidden_{t+1}$ implicitly carries information about future states $\hidden_{t+2}, \hidden_{t+3}, \dots$. As a result, NextLat not only provides learning signals that are dense in the vocabulary space, but also propagates information about future tokens into earlier representations. By augmenting sparse one-hot targets with dense latent-state supervision, NextLat extracts more learning signal from each training sequence, leading to improved data efficiency.

\subsection{Learning to Predict Next-Latent States}
We now describe the practical implementation of NextLat, which augments standard next-token prediction with self-supervised predictions in the latent space. Our NextLat implementation operates primarily on the hidden states (i.e., the final-layer outputs) as they provide compact, fixed-dimensional vectors through which gradients can be propagated through the entire transformer efficiently. As usual, we optimize the transformer and output head for next-token prediction (\cref{eq:emission_correctness}) using the cross-entropy loss:
\begin{align*}
\mathcal{L}_\text{next-token}(\theta) 
= \mathbb{E}_{t < T} \big[- \log p_\theta (X_{t+1} \mid \hidden_{t}) \big].
\end{align*}
NextLat additionally enforces transition consistency (\cref{eq:transition_correctness}) of the hidden states by introducing a latent dynamics model $p_\psi$ that predicts the next hidden state $\hidden_{t+1}$ directly from $(\hidden_t, X_{t+1})$. For a deterministic transformer model, $\mathbb{P}(\hidden_{t+1} \mid \hidden_t, X_{t+1})$ is a Dirac distribution, and we can optimize $p_\psi$ via regression\footnote{If considering a stochastic transformer model, $p_\psi$ can be optimized through variational inference.}. Moreover, observe that an ideal latent dynamics model should admit recursive consistency: its one-step map should compose correctly across multiple steps. Let $\hat{\hidden}_{t+d} = p_\psi(\hidden_t, X_{t+1:t+d})$ denote the recursive rollout of $p_\psi$ over a $d$-step horizon using teacher-forced tokens $X_{t+1:t+d}$. We supervise all $d$ intermediate rollouts using the Smooth L1 loss:
\begin{align}\label{eq:loss_next_hidden}
\mathcal{L}_\text{next-h} (\theta,\psi;d)
= \mathbb{E}_{t} \Big[\frac{1}{d}\sum_{i=1}^d 
\mathrm{SmoothL1Loss}\big(\stopgrad[\hidden_{t+i}], \hat{\hidden}_{t+i}\big) \Big],
\end{align}
where $\stopgrad[\cdot]$ denotes the stop-gradient operator, used to prevent representational collapse in self-predictive learning \citep{ni2024bridging}\footnote{Technically speaking, the next-token prediction objective already provides grounding against representational collapse. However, in our ablations (see \autoref{app:ablations}), we empirically observe better performance when applying the stop-gradient.}. Note that belief state convergence (i.e., \cref{theorem:nhs_belief_states}) \textbf{already holds for $d=1$}. Multi-step supervision serves only to provide richer learning signal. 
To further align the semantics of predicted states $\hat{\hidden}$ with true states, we introduce a complementary KL objective enforcing agreement in token-prediction space:
\begin{align}\label{eq:loss_kl}
\mathcal{L}_\mathrm{KL} (\theta,\psi;d)
= \mathbb{E}_{t } \Big[\frac{1}{d}\sum_{i=1}^d 
D_{\mathrm{KL}}\!\left(
p_\theta^{\stopgrad}(\cdot \mid{\stopgrad[\hidden_{t+i}]})
\;\|\;
p_\theta^{\stopgrad}(\cdot \mid{\hat{\hidden}_{t+i}})
\right) \Big],
\end{align}
where the output head $p_\theta^{\stopgrad}(\cdot)$ is frozen so that gradients flow only through the latent dynamics model. This KL acts similarly to knowledge distillation~\citep{hinton2015distilling}, providing soft supervision that guides learning of $p_\psi$. It also resembles observation reconstruction in self-predictive RL \citep{subramanian2022approximate,ni2024bridging}, encouraging $\hat{\hidden}_{t+i}$ to reproduce the distribution over next observations (i.e., the output head's logits). 

\textbf{Overall Objective.}  
The final NextLat objective combines all components, minimizing the following loss:
\begin{align}\label{eq:all_losses}
\mathcal{L}_\text{NextLat}(\theta, \psi;d, \lambda_\text{next-h}, \lambda_\mathrm{KL})
= \mathcal{L}_\text{next-token}(\theta)
+ \lambda_\text{next-h}\, \mathcal{L}_\text{next-h} (\theta,\psi;d)
+ \lambda_\mathrm{KL}\, \mathcal{L}_\mathrm{KL} (\theta,\psi;d),
\end{align}
where $\lambda_\text{next-h}, \lambda_\mathrm{KL} > 0$ are scalar coefficients. Importantly, during inference, the learned transformer can decode independently; $p_\psi$ is only needed during training to shape the transformer representations. 

In the experiments that follow, we parameterize $p_\psi$ using simple MLPs as our goal is to demonstrate that NextLat yields significant performance gains over baselines even without sophisticated latent dynamics architectures. Additional implementation details and ablations of key NextLat design choices are provided in \autoref{section:more_details_nextlat} and \autoref{app:ablations}, respectively. To illustrate the simplicity of our approach, we also include a PyTorch-style pseudocode of the NextLat objective in \autoref{alg:pytorch-code}.

\subsection{Variable-Length Self-Speculative Decoding}
\begin{figure}[htbp]
  \vspace{0.5em}
  \centering
  \includegraphics[width=1\linewidth]{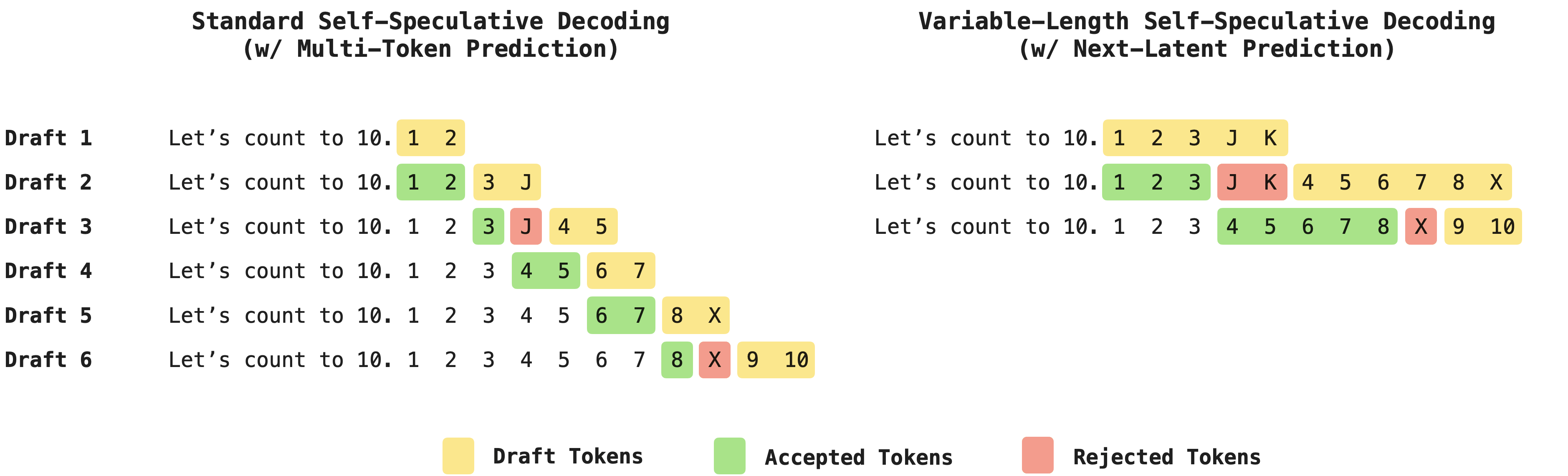} 
  \caption{Illustration comparing self-speculative decoding with multi-token prediction (MTP) vs.\ NextLat. MTP uses a fixed draft length per step (e.g., $d=2$ tokens here), whereas NextLat enables variable-length self-speculative decoding, reducing draft-verification cycles and accelerating inference.}
  \label{fig:flexible_spec_decoding}
\end{figure}
In \cref{fig:flexible_spec_decoding}, we illustrate the difference in self-speculative decoding capabilities between MTP methods and NextLat\footnote{The figure illustrates the potential for adaptive draft lengths with NextLat. In the scope of this work, however, we do not implement adaptive drafting strategies; once selected, the draft length remains static throughout decoding. Instead, we exploit the ability to vary these fixed draft lengths to extend beyond the training horizon $d$ in order to maximize inference speedup.}. MTP models are trained to predict the next $d$ tokens and are therefore limited to drafting at most $d$ tokens per draft–verification cycle. In contrast, even when trained with just $d=1$ (i.e., no multi-step supervision), NextLat can recursively compose predictions via its latent dynamics: 
\begin{align*}
    \hidden_t  \xrightarrow[\text{decode token}]{p_\theta} X_{t+1} \xrightarrow[\text{update state}]{p_\psi} \hidden_{t+1} \xrightarrow[\text{decode token}]{p_\theta} X_{t+2} \xrightarrow[\text{update state}]{p_\psi} \hidden_{t+2} \;\cdots\;,
\end{align*}
enabling variable-length drafting during self-speculative decoding. We demonstrate in \cref{subsection:finewebedu} that, even with shallow training horizons ($d=1, 2$) in the language domain, NextLat’s latent dynamics remain coherent far beyond the training horizon, enabling longer drafts and much faster inference than MTP baselines.

%% file: content/experiments.tex
\section{Experiments}

Modeling coherent latent dynamics and compact beliefs about the underlying data-generating process is fundamental to both algorithmic and human reasoning. Therefore, in this section, we evaluate NextLat on four key axes where such capabilities matter most: world modeling, reasoning, planning, and language modeling.

Our baseline comparisons include transformer-based belief-learning methods, i.e., BST~\citep{hu2025the} and JTP~\citep{ahn2025jtp}. Further discussions and detailed comparisons with these methods are provided in \autoref{section:discussions} and \autoref{app:belief-state}. For completeness, we also report the performances of standard next-token prediction (GPT) and multi-token prediction (MTP). The MTP baseline follows the implementation of \citet{gloeckle2024}, and we follow the \texttt{nanoGPT} codebase \citep{Karpathy2022} for our decoder-only transformer implementations. Hereafter, we use the term ``horizon'' to refer to the multi-step prediction horizon $d$ in JTP, MTP and NextLat, and we match horizon across these methods in all experiments to ensure fair comparisons. For specific experiment details such as hyperparameters, evaluation procedure, etc., please refer to \autoref{section:experiment_details}.

\subsection{World Modeling}
\label{subsection:manhattan_exp}

\begin{table}[htbp]
  \centering
  \small
  \begin{tabular}{lccccc}
    \toprule
    & \begin{tabular}{@{}c@{}}Next-Token \\ Test\end{tabular} $(\uparrow)$ & \begin{tabular}{@{}c@{}}Valid \\ Trajectories\end{tabular} $(\uparrow)$ & \begin{tabular}{@{}c@{}}Sequence \\ Compression\end{tabular} $(\uparrow)$ & \begin{tabular}{@{}c@{}}Effective \\ Latent Rank\end{tabular} $(\downarrow)$ & \begin{tabular}{@{}c@{}}Detour \\ Robustness\end{tabular} $(\uparrow)$ \\
    \midrule
    GPT & 100\% & 97.0\% & 0.65 & 160.1 & 85.0\% \\
    MTP &  100\% & 98.1\% & 0.64 & 57.7  & \textbf{95.0\%} \\
    JTP & 100\% & 97.1\% & 0.32 & 215.8 & 87.0\% \\
    NextLat & 100\% & \textbf{98.7\%} & \textbf{0.71} & \textbf{52.7} & \textbf{95.0\%} \\
    \midrule
    True world model & 100\% & 100\% & 1.00 & --- & 100\% \\
    \bottomrule
  \end{tabular}
  \caption{Comparison of GPT, MTP, JTP, and NextLat trained on Manhattan taxi rides against the true world model across several metrics.} 
  \label{tab:manhattan_results}
\end{table}
\citet{vafa2024evaluating} introduced a dataset of turn-by-turn taxi rides in Manhattan, where the true world model (i.e., the city’s street map) is visually interpretable. Their study revealed that transformers trained on such trajectories can achieve near-perfect next-token accuracy, yet their internal maps remain incoherent; they reconstruct streets with impossible orientations and even flyovers above other roads.

\paragraph{Setup.} We use the random walks dataset from \citet{vafa2024evaluating}, which consists of random Manhattan traversals (91M sequences, 4.7B tokens) between taxi pickup and dropoff points. 
Models are trained for 6 epochs (vs.\ 1 epoch in their study) as we observe performance does not converge within a single epoch. Due to its high computational cost, BST is excluded from this benchmark. For JTP, MTP, and NextLat, we set the multi-step prediction horizon at $d=8$. 
We evaluate world-modeling performance using five comprehensive metrics:

\begin{enumerate}
    \item Next-Token Test: Percentage of top-1 token predictions corresponding to legal turns under teacher-forcing on in-distribution validation sequences.
    \item Valid Trajectories: Percentage of valid traversals for out-of-distribution (OOD) pickup--dropoff pairs.
    \item Sequence Compression: Percentage of cases where the model produces identical continuations when prompted with two different traversals arriving at the same state and sharing the same destination.
    \item Effective Latent Rank: Effective rank/dimension of hidden states measured as the exponentiated Shannon entropy of the normalized singular values \citep{roy2007effective}; lower values indicate better compression.
    \item Detour Robustness: Percentage of valid traversals for OOD pickup--dropoff pairs when we substitute the model’s top-1 prediction with random detours (legal turns) 75\% of the time.
\end{enumerate}

The results are shown in \cref{tab:manhattan_results}. More details on training and evaluation are provided in \autoref{section:experiment_details}. We also refer motivated readers to \citet{vafa2024evaluating} for further explanations of the next-token test, valid trajectories, sequence compression, and detour robustness metrics.

\begin{figure}[htbp]
     \centering
     \begin{subfigure}[b]{0.23\textwidth}
         \centering
         \includegraphics[width=\textwidth, keepaspectratio]{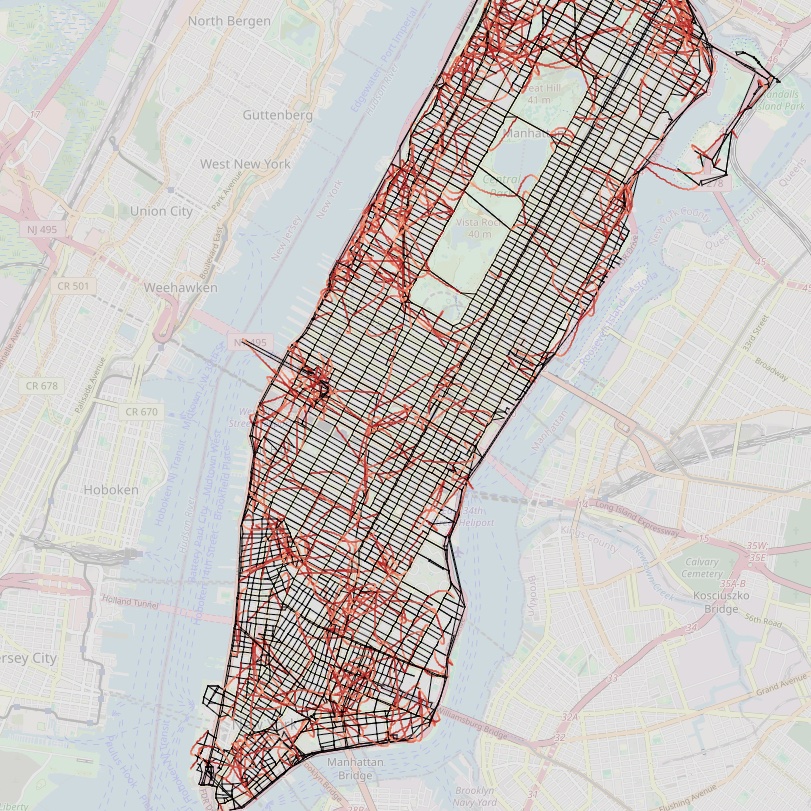}
         \caption{GPT}
         \label{fig:GPT_manhattan}
     \end{subfigure}
     \hfill
     \begin{subfigure}[b]{0.23\textwidth}
         \centering
         \includegraphics[width=\textwidth, keepaspectratio]{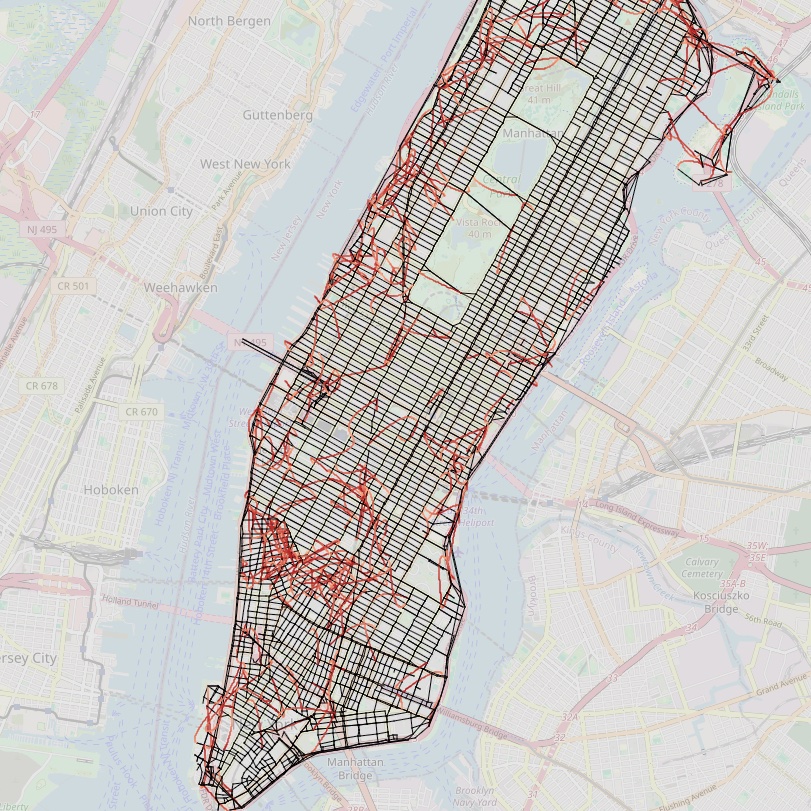}
         \caption{MTP}
         \label{fig:MTP_manhattan}
     \end{subfigure}
     \hfill
     \begin{subfigure}[b]{0.23\textwidth}
         \centering
         \includegraphics[width=\textwidth, keepaspectratio]{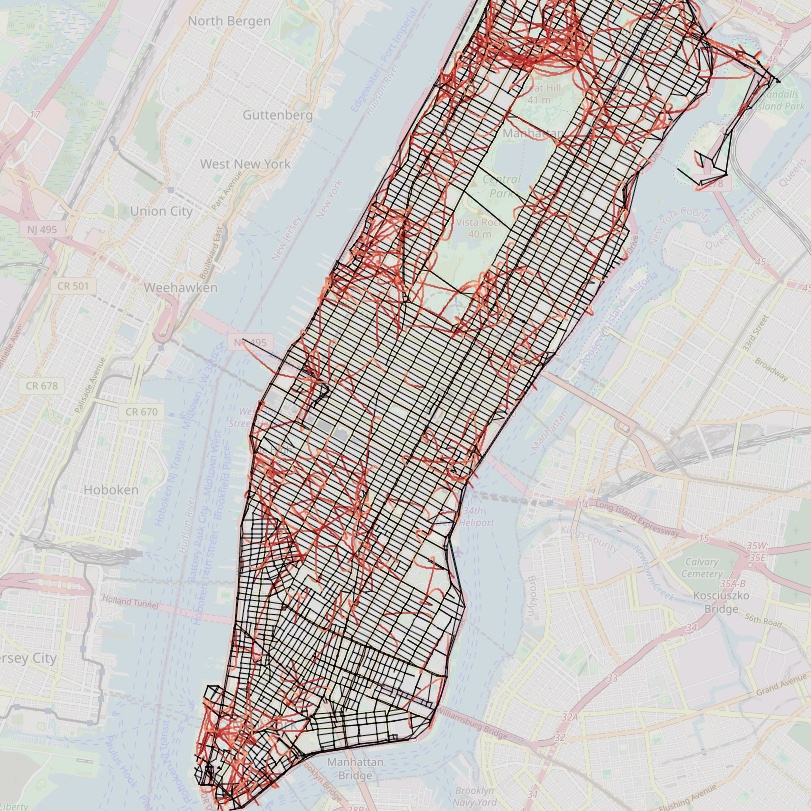}
         \caption{JTP}
         \label{fig:JTP_manhattan}
     \end{subfigure}
     \hfill
     \begin{subfigure}[b]{0.23\textwidth}
         \centering
         \includegraphics[width=\textwidth, keepaspectratio]{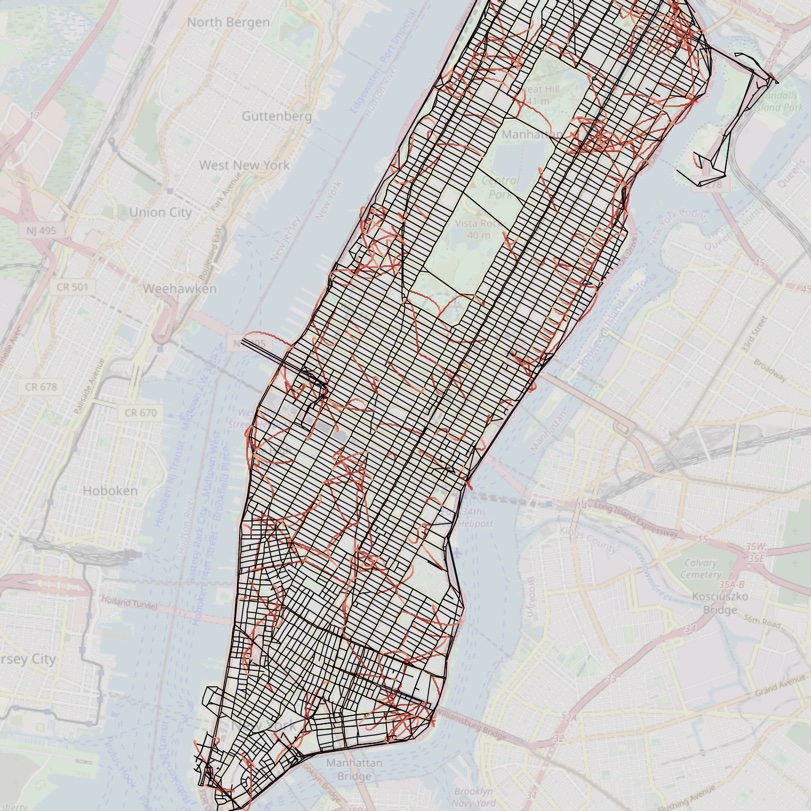}
         \caption{NextLat (ours)}
         \label{fig:SPLAT_manhattan}
     \end{subfigure}
     \caption{Reconstructed maps from transformers trained on Manhattan taxi rides using different objectives.}
     \vspace{-1em}
     \label{fig:manhattan_maps_full}
\end{figure} 

\paragraph{Results.} Similar to the original study, all models achieved 100\% accuracy on the next-token test. However, next-token accuracy is a limited diagnostic and cannot meaningfully assess the quality of a model’s learned world model. In \cref{fig:manhattan_maps_full}, we visualize each model's internal map using the reconstruction algorithm proposed by \citet{vafa2024evaluating}. Visibly, the transformer trained with NextLat exhibits an internal map more consistent with the true world model. Although not perfect, its inconsistencies (red edges) are sparse and mostly local. Beyond this qualitative evidence, NextLat consistently outperforms all baselines across all metrics. On the trajectory validity and detour robustness metrics, NextLat demonstrates the strongest generalization to OOD pickup--dropoff pairs, even when random detours are introduced.

Next, we analyze the compactness of the learned world models using two compression metrics. A model that accurately captures the underlying states and transitions should assign identical continuations to trajectories that end in the same state (i.e., an intersection in Manhattan). By this criterion, NextLat achieves the \textbf{highest sequence compression of 0.71}. The true Manhattan graph comprises only 4,580 intersections and 9,846 edges, and therefore an effective world model should require only a modest latent dimensionality. Indeed, NextLat has the \textbf{lowest effective latent rank of 52.7---over 3x smaller than GPT's}. The combination of stronger planning performance and more compact latent representations reinforces the view that NextLat, by promoting belief state representations and coherent latent dynamics, enables transformers to learn substantially better world models---ones that are both accurate in their predictive structure and efficient in their internal representation of the environment.

\subsection{Reasoning}
\label{subsection:countdown_exp}

\begin{wraptable}{r}{0.4\textwidth}
      \centering
      \begin{tabular}{llc}
        \toprule
        \textbf{Model} & \textbf{Horizon ($\mathbf{d}$)} & \textbf{Accuracy (\%)} \\ 
        \midrule
        GPT   & -- &    33.1   \\
        \midrule
        BST   & -- &   42.3    \\
        \midrule
        \multirow{3}{*}{MTP} & 1 &  39.2  \\
                             & 4 &  49.7   \\
                             & 8 &  57.3    \\
        \midrule
        \multirow{3}{*}{JTP} & 1 &   39.0    \\
                             & 4 &   49.4    \\
                             & 8 &   55.0    \\
        \midrule
        \multirow{3}{*}{NextLat} & 1 &   54.8   \\
                             & 4 &   \underline{57.6}   \\
                             & 8 &   \textbf{58.7}  \\
        \bottomrule
      \end{tabular}
      \caption{Performance on Countdown. Best result is \textbf{bolded}, and second best is \underline{underlined}.}
      \label{tab:countdown_results}
\end{wraptable}

Countdown \citep{countdown2025} is a mathematical reasoning task and a generalized version of the Game of 24, which even frontier models such as GPT-4 \citep{achiam2023gpt} have struggled with, achieving 4\% by default \citep{gameof242023}. The goal of the task is to combine a set of given numbers with basic arithmetic operations $(+, -, \times, \div)$ to obtain a target number. For example, given the numbers $\{90,8,20,50\}$, the target number $24$ can be obtained using the following sequence of equations: $90\times8=720,\:50-20=30,\:720\div30=24$. Countdown poses a difficult combinatorial search problem due to its large branching factor and the need to efficiently explore the solution space to reach the target number.

\paragraph{Setup.} Following \citet{gandhi2024stream}, we generate 500k training problems with target numbers ranging from 10 to 100 and reserve 10\% of the targets for out-of-distribution evaluation. During both training and testing, we insert eight `pause tokens' \citep{goyal2023think} after the target number, allowing models additional computation to plan before generating a solution. Performance is measured as the percentage of 10k test problems for which a model produces a valid sequence of equations that correctly reaches the target number. The reported results average over three random seeds per baseline. 

\begin{wrapfigure}{r}{0.4\textwidth}
  \centering
  \vspace{1em}
  \includegraphics[width=\linewidth]{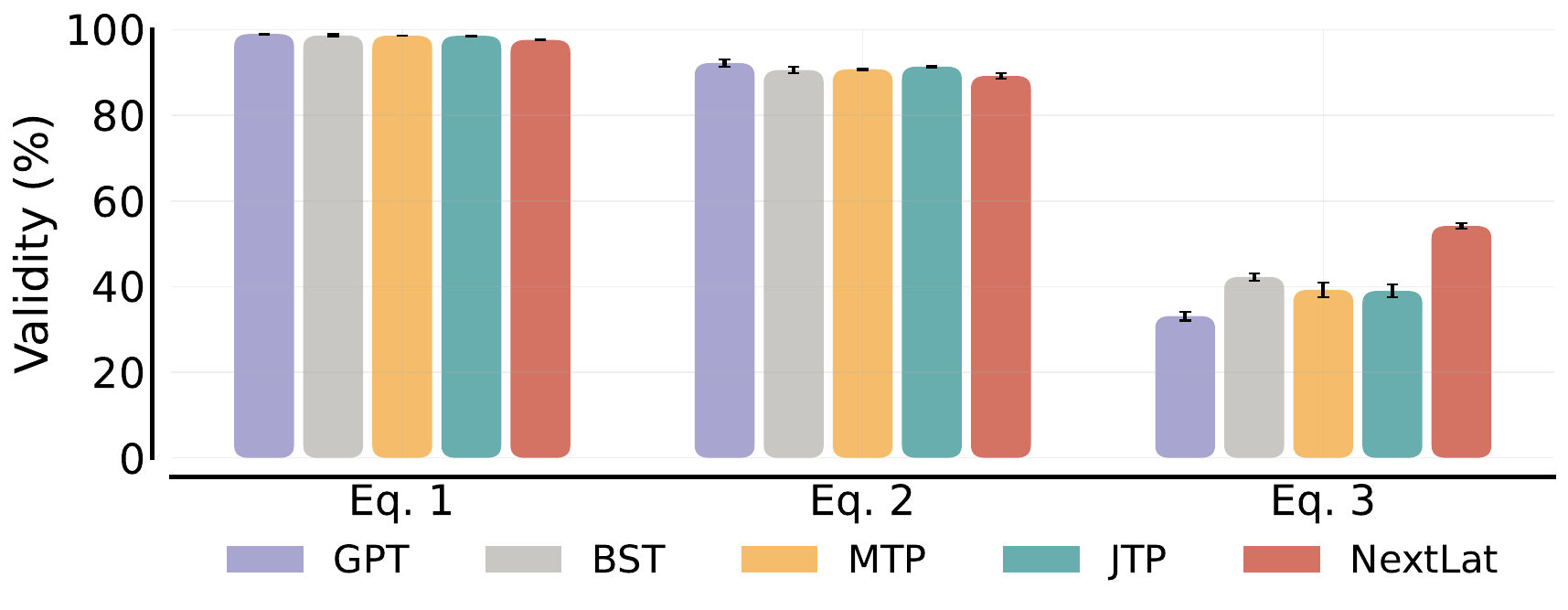}
  \caption{Validity of equations (i.e., LHS = RHS) generated on Countdown. All models in this plot use $d=1$.}
  \label{fig:countdown_validity}
\end{wrapfigure}
\paragraph{Results.} As shown in \cref{tab:countdown_results}, NextLat consistently outperforms all baselines in Countdown. Notably, even with a shallow supervision horizon of $d=1$, NextLat substantially surpasses MTP and JTP trained with the same horizon (>38\% improvement). To better understand this gap, we analyzed the equations generated by each model and evaluated their equation validity, i.e., whether the computed left-hand side equals the right-hand side. As shown in \cref{fig:countdown_validity}, most calculation errors occur in the final equation (Eq. 3). This indicates that, lacking planning capability, the model realizes only at the end that it cannot reach the goal. Unable to revise earlier missteps, it forces an invalid final equation to match the desired outcome---a behavior termed \emph{the regretful compromise} by \citet{ye2025beyond}. NextLat demonstrates the \textbf{strongest lookahead planning}: even with $d=1$, it achieves substantially higher mean validity in the final equation ($54.8$\%) compared to the next best baseline ($42.3$\%). This suggests that the latent-state prediction objective may help the model anticipate long-range dependencies and form globally consistent plans, reducing the tendency to make myopic errors. NextLat also achieves leading performance with horizons $d=4$ and $8$.

\subsection{Planning}
\label{section:planning}

\begin{figure}[htbp]
    \centering
    \vspace{-1em}
    \begin{minipage}[b]{0.7\linewidth}
        \begin{figure}[H]
            \centering
            \includegraphics[width=1.0\linewidth]{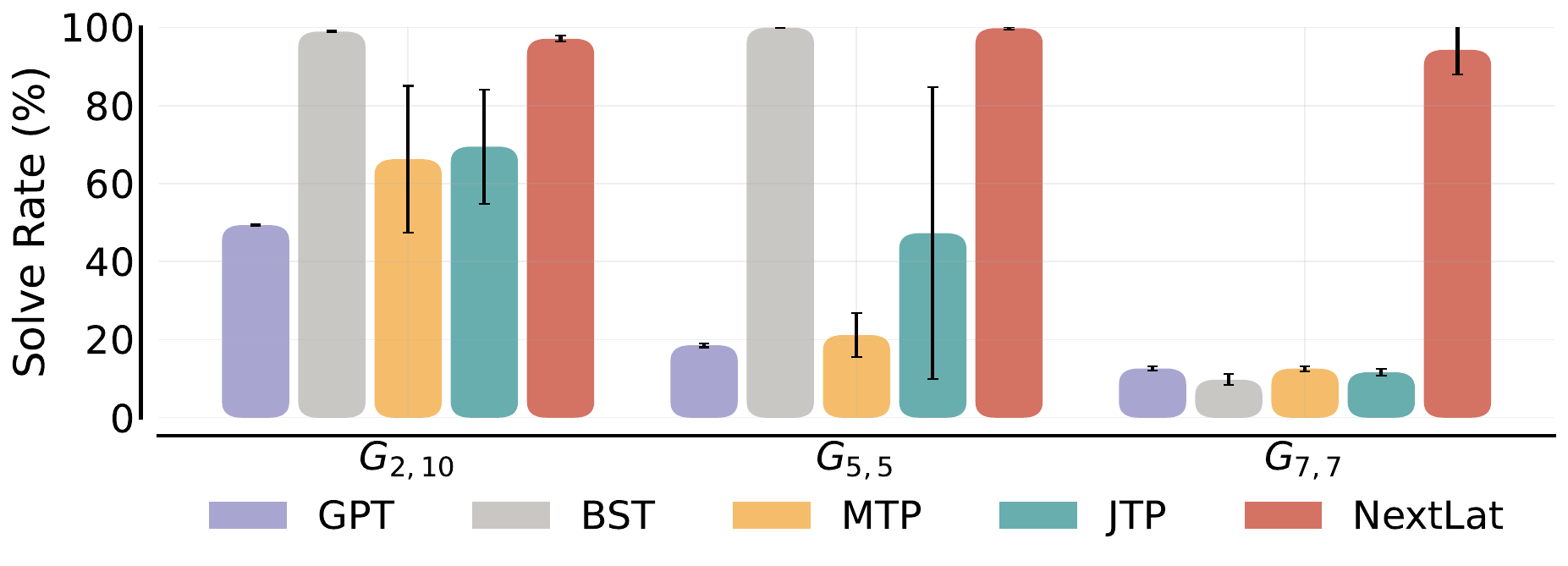}
            \caption{Accuracy on Path-Star graph task. Unlike the baselines, NextLat maintains close to 100\% solve rate for all graphs.}
            \label{fig:stargraph_result}
        \end{figure}
    \end{minipage}
    \hfill
    \begin{minipage}[b]{0.25\linewidth}
        \hfill
        \begin{figure}[H]
            \centering
            \includegraphics[width=1.0\textwidth]{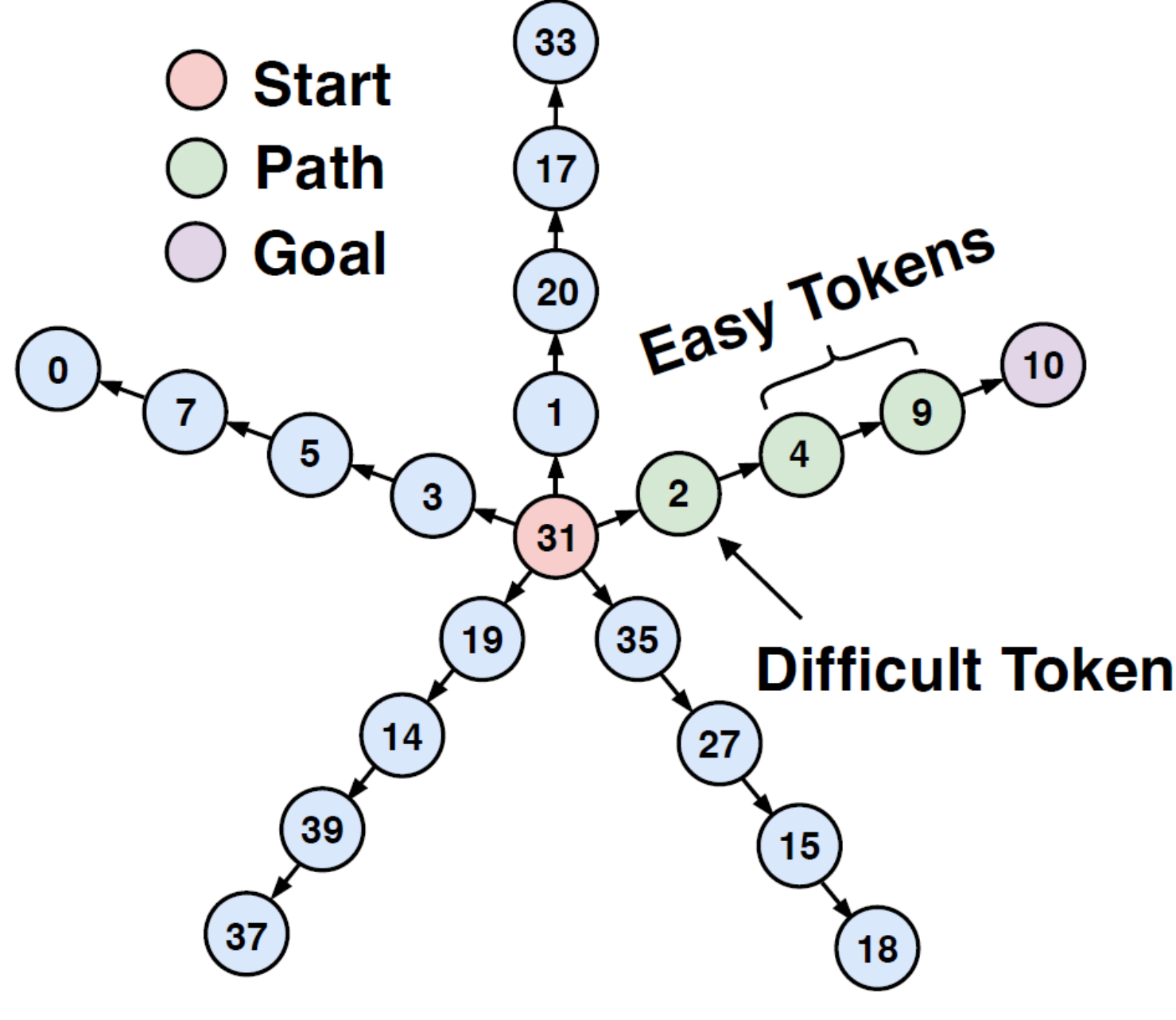}
            \caption{Illustration of a $G_{5,5}$ Path-Star graph \citep{bachmann24a}.}
            \label{fig:stargraph}
        \end{figure}
    \end{minipage}
\end{figure}

A Path-Star graph \citep{bachmann24a} $G_{k,\ell}$ consists of a center node and $k$ disjoint arms, each consisting of $\ell-1$ nodes. \cref{fig:stargraph} depicts an instantiation of the $G_{5,5}$ topology. A training instance is a tokenized sequence that contains the edge list, the start and end nodes, and the correct path from start to end. This task represents a minimal instance of lookahead planning, a core capability underlying more complex behaviors such as storytelling. Yet, despite its apparent simplicity, next-token prediction models struggle to solve it.

\paragraph{Setup.} Following \citet{bachmann24a}, we generate 200k training samples and set $N=100$, such that node values in each graph are randomly drawn from $\{1, \dots, 100\}$. For MTP, JTP, and NextLat, we set the multi-step prediction horizon to $d = \ell - 2$, ensuring that the target (end) node lies within the multi-step prediction horizon of the center node. We evaluate performance across three graph configurations: $G_{2,10}$, $G_{5,5}$, and $G_{7,7}$ across five random seeds per baseline.

\paragraph{Results.} As shown in \cref{fig:stargraph_result}, NextLat maintains close to 100\% solve rate for all topologies of the Path-Star graphs. BST, while able to solve $G_{2,10}$ and $G_{5,5}$, begins to fail at the larger graph $G_{7,7}$. Note that our results differ from those presented in \citet{hu2025the} (BST) and \citet{ahn2025jtp} (JTP). In their setup, they use a much smaller problem setting of $N=50$, and generate a fresh batch of graphs every iteration. On the other hand, we use the original (and more difficult) setup which has a fixed sample size of 200k and $N=100$. 

The Path-Star graph task is specifically designed to reveal the myopic behavior of teacher-forced next-token prediction models, which tend to exploit local shortcuts instead of learning to perform lookahead planning necessary to solve the task. This phenomenon, termed the \emph{Clever Hans cheat} \citep{bachmann24a}, is related to the difficulty of learning parity \citep{hu2025the} and has motivated methods such as BST and JTP that attempt to mitigate shortcut learning through multi-token predictions. However, these approaches operate in token space, making them still susceptible to local $n$-gram regularities that do not capture the underlying transition structure required for long-horizon planning. In contrast, NextLat performs prediction in latent space, enforcing recurrent transition consistency at the representation level. NextLat’s success across all graph configurations, unlike MTP, JTP, and BST, suggests that latent-space prediction better avoids shortcut learning and yields more generalizable solutions. Furthermore, given the limited training samples (i.e., 200k) in this task, data efficiency is crucial. NextLat's strong performance in this low-data regime suggests that it extracts more useful supervision from each training sequence than token-prediction baselines, consistent with our analysis in \cref{subsection:why_nextlat}.

\begin{figure}[htbp]
    \centering
    \includegraphics[width=1.0\linewidth]{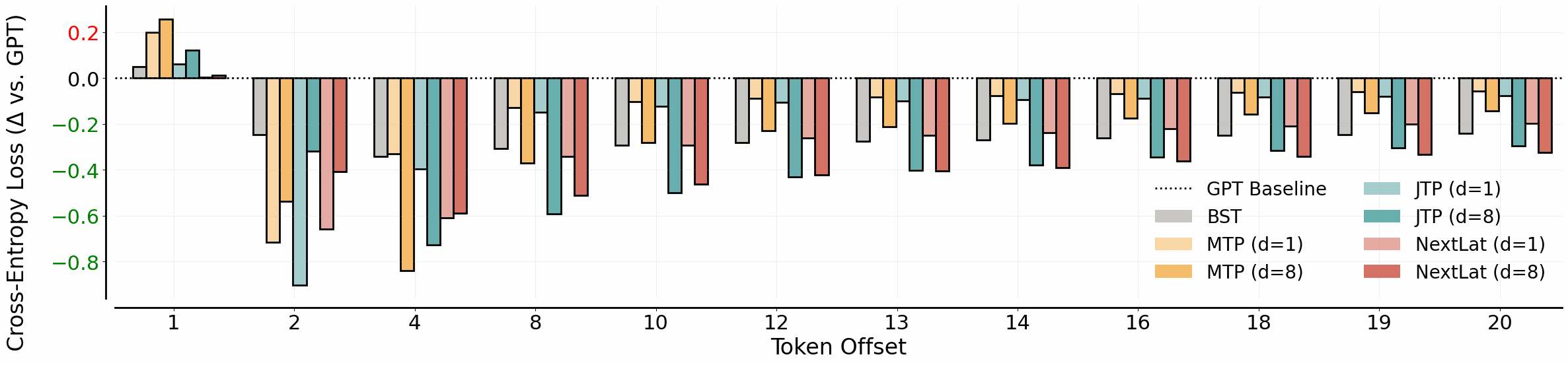}
    \caption{Cross-entropy loss difference relative to GPT, obtained from linear probes trained on frozen hidden states to predict tokens at varying token offsets (x-axis values) ahead. Lower values indicate better predictive performance.}
    \label{fig:tinystories_result}
\end{figure}
Next, we compare the models on TinyStories \citep{eldan2023tinystories}, a dataset consisting of synthetic short stories. Storytelling is inherently a long-horizon planning problem; a coherent narrative requires maintaining persistent entities, tracking causal relationships, resolving conflicts and delayed resolutions, and satisfying narrative constraints across many timesteps. Generating such sequences therefore depends not only on next-token prediction, but on belief-state--like abstractions that encode information predictive of future story trajectories.  

\paragraph{Setup.} Following \citet{hu2025the}, we tokenize the dataset of 2.7 million stories into a vocabulary of 1,000 tokens and construct training sequences of length 256. All models are trained for 100k steps, which is sufficient for convergence. We include comparisons of transformers trained using MTP, JTP, and NextLat with multi-step prediction horizons $d \in \{1, 8\}$. After training, we freeze the model parameters and train 20 independent linear probes, one per token offset, to predict tokens at offsets $1,\dots,20$ steps ahead from the hidden states of the frozen models using the same dataset. This allows us to assess whether the models’ representations encode information predictive of future tokens, or just local token correlations.

\paragraph{Results.} We plot the difference in probe performance relative to probes trained on GPT’s hidden states in \cref{fig:tinystories_result}. For clarity, we display only selected token offsets here (see \cref{fig:tinystories_results_full20} in the appendix for full results). Observe that the additional token-level prediction objectives in BST, MTP, and JTP consistently cause significant degradation in next-token prediction (i.e., token offset = 1). Moreover, probe performance on JTP and MTP representations declines sharply with increasing token offset. This indicates that these multi-token prediction models, lacking guarantees of learning belief states, could encode information useful only for short-horizon prediction. In contrast, NextLat matches GPT’s next-token performance across both $d \in \{1, 8\}$ and exhibits the strongest long-range predictive capability (up to 20 tokens ahead) for both $d=1$ and $d=8$. These results suggest that NextLat’s latent-state objective induces belief-like representations that encode predictive information about future events---an ability essential for maintaining coherence in long-range narrative generation tasks like TinyStories.

\subsection{Language Modeling}
\label{subsection:finewebedu}
\begin{table}[H]
\centering
\resizebox{\columnwidth}{!}{%
\begin{tabular}{l|ccc|cccccccccc}
\toprule
\textbf{Model} 
& \begin{tabular}{@{}c@{}}\textbf{FW-Edu} \\ ppl $\downarrow$\end{tabular} 
& \begin{tabular}{@{}c@{}}\textbf{Wiki.} \\ ppl $\downarrow$\end{tabular} 
& \begin{tabular}{@{}c@{}}\textbf{LAMB.} \\ ppl $\downarrow$\end{tabular}  
& \begin{tabular}{@{}c@{}}\textbf{LAMB.} \\ acc $\uparrow$\end{tabular} 
& \begin{tabular}{@{}c@{}}\textbf{PIQA} \\ acc $\uparrow$\end{tabular}
& \begin{tabular}{@{}c@{}}\textbf{HellaS.} \\ acc $\uparrow$\end{tabular}
& \begin{tabular}{@{}c@{}}\textbf{Wino.} \\ acc $\uparrow$\end{tabular}
& \begin{tabular}{@{}c@{}}\textbf{ARC-e} \\ acc $\uparrow$\end{tabular}
& \begin{tabular}{@{}c@{}}\textbf{ARC-c} \\ acc $\uparrow$\end{tabular}
& \begin{tabular}{@{}c@{}}\textbf{SIQA} \\ acc $\uparrow$\end{tabular}
& \begin{tabular}{@{}c@{}}\textbf{SciQ} \\ acc $\uparrow$\end{tabular}
& \textbf{Avg.} \\
\midrule

GPT 
& \textbf{10.52} & \textbf{17.93} & 20.26 & 42.07 & 73.45 & \textbf{58.79} & \underline{60.46}
& 68.18 & 39.16 & 42.32 
& 86.10
& 58.82 \\

JTP (d=1) 
& 11.08 & 19.28 & 21.88 & 41.35 & \textbf{74.92} & 57.43 & 58.64
& 68.73 & 39.25 & 42.99 
& \underline{87.30}
& \underline{58.83} \\

JTP (d=2) 
& 11.18 & 19.60 & 22.11 & 41.37 & 73.34 & 56.84 & 59.98
& 68.86 & 38.57 & \textbf{43.35}
& 86.70
& 58.63 \\

MTP (d=1) 
& 10.90 & 18.82 & 20.23 & 41.26 & \underline{74.32} & 58.05 & \textbf{60.54}
& 68.52 & 38.91 & 42.84 
& 85.40
& 58.73 \\

MTP (d=2) 
& 11.00 & 18.61 & \underline{18.34} & \underline{43.43} & 72.80 & 57.92 & 59.35
& 68.35 & 39.08 & 41.97 
& 86.60
& 58.69 \\

\midrule
\textbf{NextLat (d=1)} 
& \underline{10.83} & \underline{18.39} & 19.77 & 41.08 & 73.07 & \underline{58.35} & 59.27
& \underline{69.65} & \underline{39.68} & \underline{43.24} 
& 86.00
& 58.79  \\

\textbf{NextLat (d=2)} 
& 10.88 & 18.44 & \textbf{17.83} & \textbf{43.86} & 73.61 & 57.79 & 59.20
& \textbf{69.74} & \textbf{40.10} & 41.91 
& \textbf{87.50}
& \textbf{59.21} \\

\bottomrule
\end{tabular}
}
\caption{Downstream language modeling evaluation on 1.3B-parameter models trained on 100B FineWeb-Edu tokens. Best scores are in bold and second-best are underlined.}
\label{tab:lm_eval_results}
\end{table}

\paragraph{Setup.} We pretrain 1.3B-parameter models on 100B tokens from the FineWeb-Edu dataset \citep{penedo2024fineweb}, excluding BST due to its high computational cost. After pretraining, we use the LM Evaluation Harness \citep{eval-harness} to evaluate the zero-shot accuracy of the models on multiple-choice language modeling benchmarks. We also evaluate the self-speculative decoding performance of the multi-step prediction models, i.e., JTP, MTP, and NextLat, across Wikipedia, Books, Code, and Math domains. For each dataset, we sample 1024 prompts of length 512 tokens and generate 512-token continuations using the speculative sampling algorithm of \citet{leviathan2022fast}. For each model, we report the average number of accepted tokens per drafting step, as well as the inference speedup relative to standard autoregressive sampling from the base transformer, measured on $8\times$NVIDIA B200 GPUs. For multi-step prediction models, we primarily focus our analysis on training with horizons of $d=1$ and $d=2$. MTP and JTP operate in token space and are therefore usually limited to speculative decoding within the fixed multi-token horizon used during training. Our aim is to highlight the advantages of NextLat’s variable-length self-speculative decoding, which enables drafting beyond the training horizon. To this end, we vary NextLat’s speculative draft length between 2 and 10 tokens and report the highest inference speedup achieved for each domain.

\begin{table}[ht]
\centering
\resizebox{\columnwidth}{!}{%
\begin{tabular}{lcc|cc|cc|cc}
\toprule
 & \multicolumn{2}{c}{\textbf{Wikipedia}} 
 & \multicolumn{2}{c}{\textbf{Books}} 
 & \multicolumn{2}{c}{\textbf{Code}}
 & \multicolumn{2}{c}{\textbf{Math}} \\
\cmidrule(lr){2-3} \cmidrule(lr){4-5} \cmidrule(lr){6-7} \cmidrule(lr){8-9}
\textbf{Model} 
& Speedup & Accepted Tokens 
& Speedup & Accepted Tokens 
& Speedup & Accepted Tokens
& Speedup & Accepted Tokens \\
\midrule
JTP (d=1) & 1.46 & 0.96 & 1.47 & 0.97 & 1.47 & 0.98 & 1.46 & 0.97  \\
JTP (d=2) & 1.88 & 1.84 & 1.90 & 1.89 & 1.88 & 1.85 & 1.89 & 1.86 \\
MTP (d=1) & 1.38 & 0.91 & 1.39 & 0.95 & 1.40 & 0.97 & 1.39 & 0.95 \\
MTP (d=2) & 1.68 & 1.72 & 1.72 & 1.83 & 1.75 & 1.91 & 1.72 & 1.84 \\
\midrule
\textbf{NextLat (d=1)} & \underline{2.68} & \underline{3.52} & \underline{2.72} & \underline{3.64} & \underline{2.29} & \underline{2.66} & \underline{2.30} & \underline{2.72} \\
\textbf{NextLat (d=2)} & \textbf{3.21} & \textbf{4.59} & \textbf{3.32} & \textbf{4.86} & \textbf{2.38} & \textbf{2.83} & \textbf{2.87} & \textbf{3.94} \\
\bottomrule
\end{tabular}%
}
\caption{Relative speedup and average accepted tokens per drafting step over diverse domains. Note that ``Accepted Tokens'' excludes the next-token prediction which is always accepted.}
\label{tab:spec_eval_results}
\end{table}

\paragraph{Results.} In \cref{tab:lm_eval_results}, we present the language modeling perplexity (ppl) and zero-shot accuracy (acc) of the pretrained models across several benchmarks. Consistent with observations in \citet{gloeckle2024}, we do not observe significant improvements in multiple-choice task accuracy over standard next-token training (GPT) when using multi-token prediction objectives. NextLat ($d=2$) is the \textbf{only multi-step prediction objective that meaningfully improves average accuracy over GPT (59.21 vs.\ 58.82)}, but these improvements are modest and not consistent across tasks. Larger model sizes might be necessary to see more significant improvements. In contrast, MTP and JTP do not benefit zero-shot accuracy over next-token training. Notably, NextLat \textbf{better preserves next-token perplexity compared to MTP and JTP} across FineWeb-Edu, Wikitext, and LAMBADA. This is consistent with our earlier observations on TinyStories in \cref{section:planning}. Preserving high fidelity in next-token prediction is arguably important, as prior work shows that lower pretraining perplexity correlates with improved aggregate downstream and post-fine-tuning performance at larger model scales \citep{gadre2024languagemodelsscalereliably,zhang2026trainbeforetest}.

Next, in \cref{tab:spec_eval_results}, we show the self-speculative decoding results. Observe that the average accepted tokens per drafting step for NextLat far exceeds its training horizon $d$, indicating that the learned latent dynamics remain coherent over extended rollouts. This further highlights the strong long-range predictive capability of the induced belief state representations. Crucially, this ability to support variable-length self-speculative decoding enables NextLat to achieve \textbf{substantially higher inference speedups than MTP and JTP across all domains}. \cref{fig:fineweb_val_spec_results} further shows speedup and cumulative acceptance rate versus draft length on the FineWeb-Edu validation set. Speedup increases sublinearly with draft length, reaching up to $3.3\times$, and fully valid (i.e., all tokens accepted) drafts persist even at length 10. This demonstrates a clear benefit to drafting tokens beyond the training horizon with NextLat. In \cref{subsubsection:d4_results}, we extend our comparisons to include JTP and MTP trained with larger horizons ($d=4$). Note that the training cost of JTP and MTP increases substantially with larger $d$ (see \cref{tab:fineweb_compute}), making JTP ($d=4$) and MTP ($d=4$) significantly more expensive to train than NextLat ($d=1,2$). Even with longer multi-token supervision, these baselines still fail to surpass NextLat in zero-shot accuracy and speculative decoding performance. This highlights a key advantage of NextLat's variable-length speculative decoding: it enables long speculative drafts while requiring training only at shallow, computationally efficient horizons.

\begin{figure}[H]
     \centering
     \begin{subfigure}[b]{0.27\textwidth}
         \centering
         \includegraphics[width=1.0\textwidth, keepaspectratio]{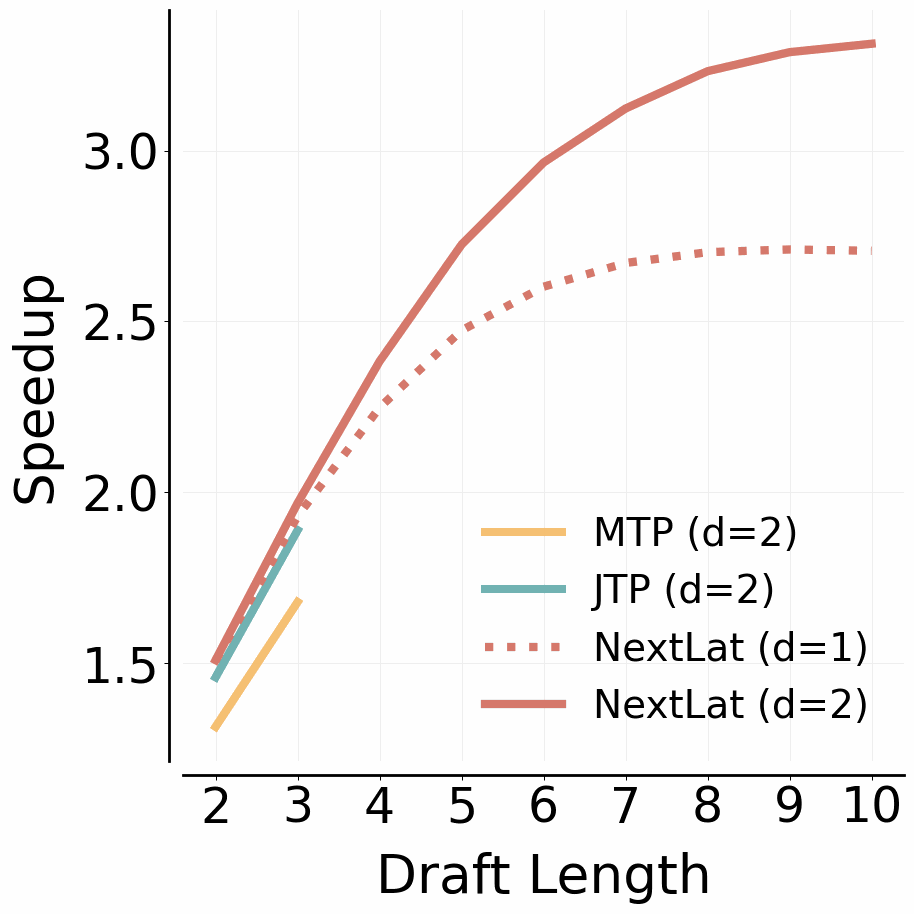}
         \caption{}
     \end{subfigure}
     \hfill
     \begin{subfigure}[b]{0.63\textwidth}
         \centering
         \includegraphics[width=1.0\textwidth, keepaspectratio]{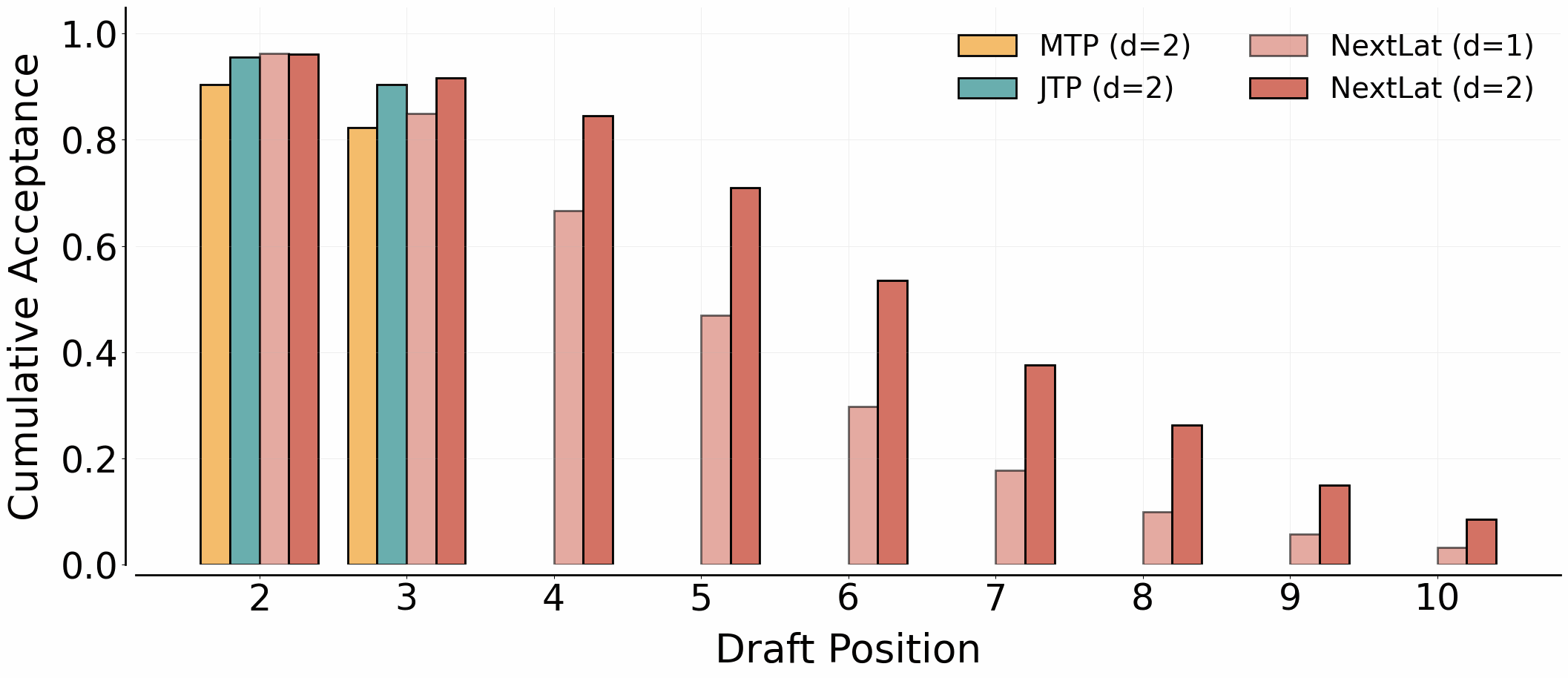}
         \caption{}
     \end{subfigure}
     \caption{(a) Inference speedup ratio of JTP, MTP, and NextLat on the FineWeb-Edu validation set. (b) Cumulative acceptance rate of draft tokens (product of per-position acceptance probabilities), indicating how likely the entire speculative draft remains valid (i.e., all tokens accepted) as draft length increases.}
     \label{fig:fineweb_val_spec_results}
\end{figure}

%% file: content/discussions_and_comparisons.tex
\section{Discussions}
\label{section:discussions}

\subsection{Comparison Against Current Approaches}
\begin{table}[htbp]
  \vspace{0.5em}
  \centering
  \begin{tabular}{lccccc}
    \toprule
    & GPT & BST  & MTP & JTP & NextLat \\
    \midrule
    \begin{tabular}{@{}l@{}}Training parameters \\ ($d=1, 2, 8$) \end{tabular} & 1.32B & 2.57B & 1.37B / 1.42B / 1.72B & 1.34B & 1.40B \\
    Inference parameters & 1.32B & 1.32B / 2.57B & 1.32B & 1.34B & 1.32B \\
    \begin{tabular}{@{}l@{}}Training steps/second \\ ($d=1, 2, 8$) \end{tabular} & 3.09 & 0.89 & 2.80 / 2.58 / 1.70 & 3.15 / 2.92 / 2.02 & 3.09 / 2.79 / 1.73  \\
    Gradients & $O(T)$ & $O(T^2)$ & $O(Td)$ & $O(Td)$ & $O(Td)$ \\
    \bottomrule
  \end{tabular}
  \caption{Comparison of training speed, parameter count and gradient signals provided on FineWeb-Edu pretraining. For training parameter numbers and training iterations/second, we report values for prediction horizons $d=1, 2, 8$. If only a single value is shown, it means that the values are the same for all horizons. All training speeds are measured on a single NVIDIA B200 GPU with batch size of 33k tokens.} 
  \vspace{0.5em}
  \label{tab:fineweb_compute}
\end{table}
In this section, we discuss the advantages of NextLat compared to prior transformer learning approaches, focusing discussions on belief-learning methods, i.e., BST and JTP. We provide brief descriptions of these methods and their training objectives in \autoref{app:belief-state}. Motivated readers interested in further details beyond the scope of this work may refer to \citet{hu2025the} and \citet{ahn2025jtp}. \cref{tab:fineweb_compute} summarizes the training and inference parameters, training speed (in iterations per second), and gradient signal characteristics for each method on FineWeb-Edu pretraining, providing context for the discussions that follow.

\paragraph{Computational Costs.} While BST benefits from $O(T^2)$ gradient signals per token sequence, this is a double-edged sword. Even with the optimized implementation of \citet{hu2025the}, training remains extremely costly because gradients must be accumulated over all $O(T^2)$ predictions of different prefix–suffix pairs. Moreover, BST trains two transformer encoders, further increasing compute and parameter cost. During inference, BST also uses both transformer encoders, one for generation and the other for scoring the likelihood of the generated sequence. On FineWeb-Edu pretraining, BST is over $3\times$ slower than NextLat in training speed (see \cref{tab:fineweb_compute}). Note that this result already uses 10\% subsampling of prefix-suffix pairs to prevent out-of-memory issues; the training speed of BST would be much slower if all pairs were used for training. In contrast, NextLat with $d=1$ incurs negligible overhead relative to GPT while achieving the same belief-state learning guarantees as BST. 

We next compare the compute costs of the multi-step prediction methods, i.e., MTP, JTP, and NextLat. The MTP implementation of \citet{gloeckle2024}, as well as other variants such as the one introduced in \citet{liu2024deepseek}, require additional transformer layers as the prediction horizon $d$ increases, whereas JTP and NextLat keep parameter counts fixed across horizons. JTP and NextLat exhibit similar training speeds for $d=1$, while MTP lags behind\footnote{JTP exhibiting a slightly faster training speed than GPT at $d=1$ is likely an artifact of \texttt{torch.compile()} due to differences in the computation graph.}. At $d=8$, JTP is substantially faster than NextLat because NextLat sequentially unrolls its latent dynamics model $p_\psi$ to compute multi-step losses, while JTP computes them in parallel. Nonetheless, this modest increase in computation for NextLat yields substantially better performance than JTP across all benchmarks. Importantly, NextLat’s sequential computation remains far more efficient than that of recurrent neural networks (RNNs), which we discuss further in \cref{section:nextlat_vs_rnn}.

\paragraph{Belief State Learning.} GPT and MTP lack any theoretical learning pressure to form belief-state representations, which means that they do not necessarily learn sufficient representations predictive of future observations. JTP can learn belief states but only under the restrictive condition that the prediction horizon satisfies $d + 1 \geq k$, where $k$ denotes the observability horizon of the underlying data-generating process (see \cref{def:k_observability} in appendix). In long-context sequence modeling, however, the underlying process is often $k$-observable for a very large $k$, rendering this condition impractical. NextLat, on the other hand, learns belief-state representations independently of $d$ and larger multi-step prediction horizons are used only to provide richer gradient signals. It also avoids the expensive $O(T^2)$ gradient computation required by BST to learn belief states. As such, NextLat stands as a simple, computationally efficient, and theoretically grounded alternative to existing belief-state learning approaches (i.e., BST and JTP) in autoregressive sequence modeling.

\paragraph{Myopic Nature of Token-level Predictions.} Token-level supervision is often myopic. Next-token prediction transformers tend to prioritize short-term dependencies, and studies have shown that early training can often resemble $n$-gram modeling, which can delay or even prevent learning of the true Markov kernel \citep{edelman2024,makkuva2025attention}. \citet{bachmann24a} further showed that the myopic nature of next-token prediction training can trap models in suboptimal local minima, undermining long-horizon planning. In our experiments, we find that adding additional token-level prediction objectives, as in BST, JTP, and MTP, not only degrades next-token performance but also fails to yield consistent gains in generalization. In contrast, NextLat, which emphasizes latent transition modeling as its primary objective, minimizes degradation in next-token performance and improves downstream generalization by encouraging the learning of predictive representations rather than shallow token-level correlations. Moreover, as discussed in \cref{subsection:why_nextlat}, latent-space supervision provides denser gradient signals than token-level supervision. Overall, NextLat offers a compelling alternative to conventional token-level prediction objectives for transformer training.

\subsection{NextLat vs.\ Recurrent Neural Networks}
\label{section:nextlat_vs_rnn}
Algorithmic reasoning requires capabilities most naturally understood through recurrent models of computation, like the Turing machine. However, strict recurrence imposes a sequential computation bottleneck at training time. NextLat introduces a recurrent inductive bias for the learned representations via latent transition prediction without turning the transformer into a strictly sequential model.

\begin{wrapfigure}{r}{0.48\textwidth}
    \centering
    \vspace{-1.5em}
    \includegraphics[width=1.0\linewidth]{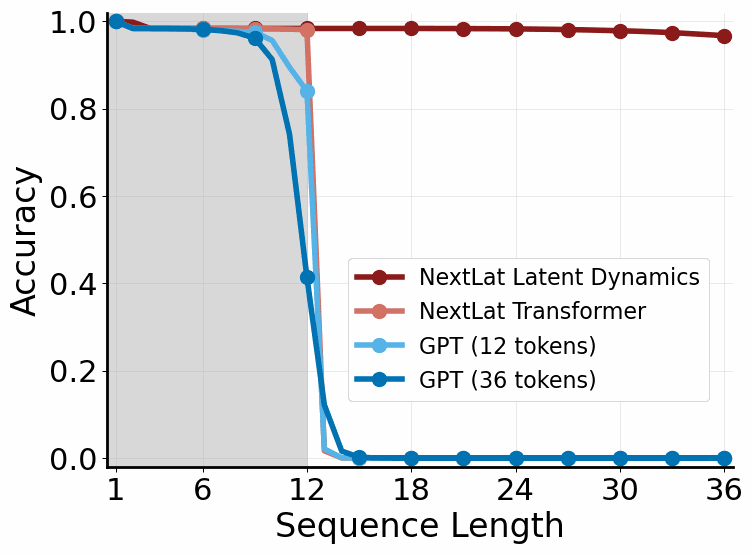}
    \caption{Length Generalization for $A_5$ word problem.}
    \vspace{-1em}
    \label{fig:a5_length_gen}
\end{wrapfigure}

\paragraph{Training RNNs without recurrence.} Intuitively, NextLat ($d=1$) can be viewed as fully parallel co-training of a transformer and a recurrent neural network (RNN). The transformer forecasts latent states across the sequence in parallel, while the latent dynamics model (the RNN) is trained using only one-step latent predictions. This co-training process does not require any sequential computation across time. 

To study this interaction, we consider the $A_5$ word problem, a difficult state-tracking task involving even permutations of five elements. The problem is $\text{NC}^1$-complete and therefore
inexpressible by constant-depth, fixed-precision transformers, which lie in $\text{TC}^0 \subset \text{NC}^1$ \citep{merrill2023parallelismtradeofflimitationslogprecision}\footnote{This assumes the widely held conjecture that $\text{TC}^0 \neq \text{NC}^1$.}. We train 2-layer transformers using next-token prediction (GPT) and NextLat ($d=1$) on the $A_5$ word problem with a sequence length of 12 tokens. We then test the transformers and the latent dynamics model (RNN) on their length generalization to sequences of length 36. At inference, the RNN only uses the transformer backbone to initialize the hidden state for the first token, after which all subsequent transitions and predictions come from the RNN alone. Importantly, the RNN also has far fewer parameters than the transformer (2.62M vs.\ 6.43M parameters).

In \cref{fig:a5_length_gen}, the transformer trained with \textbf{NextLat exhibits better state-tracking performance than GPT} within the 12-token training horizon. From a circuit complexity perspective, solving problems in $\text{NC}^1$ requires transformers with $O(\log T)$ depth \citep{merrill2023parallelismtradeofflimitationslogprecision,liu2023shortcuts}. This explains why the shallow 2-layer transformers (both GPT and NextLat) are unable to generalize beyond 12 tokens. Even GPT trained directly on 36-token sequences (``GPT (36 tokens)'') fails, consistent with previous results \citep{merrill2025illusionstatestatespacemodels}. Surprisingly, however, the latent dynamics model learned under NextLat generalizes to 36-token sequences ($>95$\% accuracy). Thus, \textbf{the co-trained RNN generalizes far beyond the expressivity limits of the underlying transformer, despite being trained entirely in parallel}. This raises an intriguing question:
\begin{displayquote}
\emph{Given inputs of length $T$, can a transformer with $O(\log T)$ depth co-train (in parallel) an RNN that generalizes to sequences of length $\gg T$? More specifically, can the co-trained RNN learn $\text{NC}^1$ computations, even though it is supervised using representations from a transformer in $\text{TC}^0$?}
\end{displayquote}
Intuitively, this seems plausible. As long as the transformer has sufficient depth to express and forecast hidden states within the $T$-length training context window, it makes sense for the latent dynamics model (RNN) to learn reusable computational circuits that extrapolate far beyond it, as empirically suggested by \cref{fig:a5_length_gen}. In this view, NextLat may offer a mechanism for partially escaping the practical limitations imposed by the \emph{parallelism tradeoff} \citep{merrill2023parallelismtradeofflimitationslogprecision}: the transformer provides efficient parallel representation learning, while the RNN focuses on learning generalizable sequential update rules. A later work, \citet{kumar2026pretrainingrecurrentnetworksrecurrence}, further illustrates the use of next-latent prediction to train RNNs without recurrence.

\paragraph{Computational Efficiency and Relation to Backpropagation Through Time.} Now, we further compare NextLat against traditional RNNs trained using backpropagation through time.
Traditional RNNs incur $O(T)$ sequential dependence during training. In contrast, NextLat adds only an additional sequential cost proportional to the rollout horizon $d \ll T$, corresponding to unrolling the latent dynamics model $p_\psi$ for $d$ steps. This allows NextLat to largely retain the transformer’s parallel training efficiency while also inheriting a recurrent inductive bias.

Conceptually, the one-step and multi-step predictions in NextLat resemble truncated backpropagation through time (TBPTT) in RNNs, with truncation windows of $1$ and $d$, respectively. However, a key distinction lies in how gradients are propagated. In TBPTT for RNNs, gradient computation is truncated beyond the chosen window, yielding biased gradient estimates that lack convergence guarantees. In contrast, the theoretical convergence of belief-state learning in NextLat, as shown in \cref{theorem:nhs_belief_states}, is \textbf{irrespective of $d$}, since it requires only one-step prediction optimality. Intuitively, NextLat performs full backpropagation through the transformer’s computation graph, where self-attention connects all past tokens. The latent dynamics model operates in an ``outer loop'', aligning the transformer's hidden states to be temporally consistent across steps. Crucially, this outer-loop supervision does not truncate gradient flow within the transformer, and therefore avoids bias with respect to $d$. Larger prediction horizons simply provide richer supervision and faster empirical convergence.

\begin{wrapfigure}{r}{0.35\textwidth}
    \centering
    \includegraphics[width=1.0\linewidth]{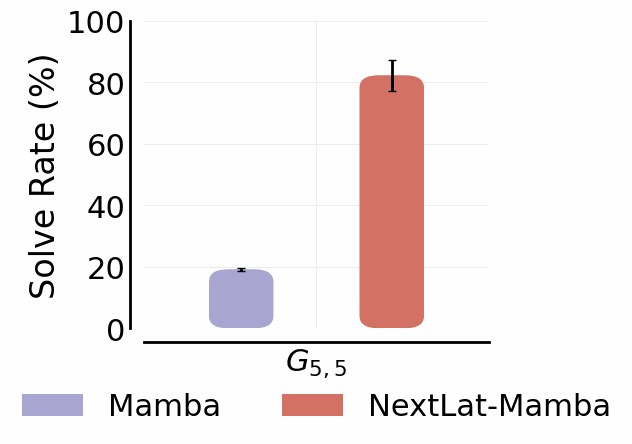}
    \caption{Accuracy on $G_{5,5}$ Path-Star graph task. NextLat improves the performance of Mamba.}
    \vspace{-1em}
    \label{fig:nextlat_mamba}
\end{wrapfigure}

\paragraph{Expressivity of the Recurrence.} 
Modern state-space models (SSMs), such as S4 and Mamba, implement efficient linear recurrent updates while remaining highly parallelizable \citep{gu2022efficientlymodelinglongsequences,gu2024mambalineartimesequencemodeling}. In contrast to SSMs, the latent dynamics model $p_\psi$ in NextLat can express nonlinear transitions. Moreover, NextLat does not explicitly perform recurrence in the forward computation. Instead, recurrent-like dynamics emerge implicitly through one-step or multi-step unrolling of $p_\psi$ and aligning successive hidden states via regression. This induces temporal consistency within the latent space without altering the transformer architecture. However, it is important to note that NextLat modifies the learned representations, not the underlying circuit complexity. The overall computational expressivity of the fixed-depth transformer trained with NextLat remains bounded within the $\text{TC}^0$ complexity class \citep{merrill2022saturated}. However, our earlier results in \cref{fig:a5_length_gen} raise the question of whether the co-trained latent dynamics model (RNN) is subject to the same $\text{TC}^0$ limitations.

\paragraph{Hybrid Attention Models.}
NextLat differs fundamentally from hybrid Transformer–SSM architectures \citep{park2024mambalearnlearncomparative,lieber2024jambahybridtransformermambalanguage,ren2025sambasimplehybridstate}, which explicitly interleave SSM and attention layers. NextLat requires no architectural changes and operates purely as an auxiliary training objective on the model’s latent representations. As such, it is broadly compatible with diverse sequence-modeling architectures. In fact, NextLat could complement hybrid models by encouraging compression and consistency in the residual attention pathways that fall outside the SSM’s recurrent structure, potentially improving representation efficiency and generalization. NextLat can even be applied to pure SSMs. Although SSMs already compress history through a linear recurrent bottleneck, NextLat adds more explicit pressure toward belief-state learning, along with the data-efficiency benefits mentioned in \cref{subsection:why_nextlat}. In \cref{fig:nextlat_mamba}, we train Mamba \citep{gu2024mambalineartimesequencemodeling} with NextLat on the $G_{5,5}$ Path-Star graph task and find that it significantly improves its solve rate.

%% file: content/conclusion.tex
\section{Limitations and Future Work}
While NextLat shows strong empirical performance, several limitations remain in our work. First, we do not explore the design space of the latent dynamics model; all experiments use simple MLPs to isolate and demonstrate the effectiveness of the core NextLat approach. More expressive architectures may further improve performance. We also do not study how the width of the hidden layers in the latent dynamics MLP affects learning, even though it effectively acts as a bottleneck that constrains belief-state capacity and may influence performance across tasks. Empirically, we observe that using smaller latent dimensions is beneficial on tasks such as Path-Star graph and Countdown.

Second, the design of the NextLat objective (e.g., stop-gradients, KL self-distillation, Smooth L1 loss) is guided largely by small-scale ablations in \autoref{app:ablations} and empirical intuition. It remains unclear whether multi-step supervision ($d>1$) and KL token-level supervision are even necessary at larger model and data scales. More systematic studies are needed to better understand how these components interact. Third, due to computational constraints, we did not evaluate against more recent or specialized MTP variants such as the one introduced in \citet{liu2024deepseek}. Finally, we did not fully exploit the variable-length nature of NextLat’s speculative decoding. In our experiments, the draft length remained fixed throughout decoding for each prompt; we only varied the static draft length between 2 and 10 tokens to identify the configuration with the highest inference speedup. We leave the exploration of more creative adaptive-length speculative decoding strategies for NextLat to future work.

On the analysis side, we do not study the structure of the learned representations under NextLat, leaving open questions about how the method shapes latent spaces. In \cref{section:nextlat_quirks}, we also highlight several quirks observed during pretraining with NextLat, such as increases in Smooth L1 loss ($\mathcal{L}_\text{next-h}$) during training and differing loss trajectories across optimizers. These observations suggest that NextLat can be sensitive to optimization dynamics. Better understanding of how to scale and parameterize the NextLat objective remains an important direction for future work.

This work represents only a preliminary study of next-latent prediction, leaving many promising directions for future research. Since the method requires no architectural changes beyond a lightweight latent dynamics model for shaping representations, an interesting direction is to apply it as a post-hoc finetuning objective for pretrained transformers. This could potentially improve reasoning, planning, and world-modeling capabilities of existing models without retraining from scratch. Moreover, because NextLat effectively organizes latent representations with recurrent-like dynamics, an interesting question is whether transformers trained with NextLat are better suited for RL post-training, where value estimation (be it implicit or explicit) benefits from such recursive ``Bellman-like'' latent structure. Finally, it would be valuable to explore richer latent architectures, such as higher-dimensional or hierarchical belief states spanning multiple layers or tokens, which may further improve long-horizon reasoning and planning.

\section{Scaling Next-Latent Prediction}
\label{section:scaling_nextlat}
All experiments in this paper parameterize the latent dynamics model $p_\psi$ with a lightweight MLP. Scaling NextLat to larger models will likely require a more expressive $p_\psi$, on par with the transformer layers used in modern multi-token prediction modules \citep{liu2024deepseek}. Since completing the experiments in this paper, we have identified several design choices that increase this expressivity while preserving belief-state learning.

\paragraph{Normalizing the inputs and outputs of $p_\psi$.} Since $\hidden_t$ is already the transformer's final post-norm output, $p_\psi$ only needs to normalize the token embedding, i.e., take $[\hidden_t;\, \mathrm{Norm}(\mathrm{Emb}(X_{t+1}))]$ as input instead of normalizing the full concatenation as currently done (see \autoref{section:more_details_nextlat}); preliminary experiments suggest this substantially improves perplexity. The normalization and MLP should match the base transformer, e.g., RMSNorm \citep{zhang2019rmsnorm} and SwiGLU \citep{shazeer2020glu}. Normalizing the output of $p_\psi$ further lets it focus on modeling the direction of the next latent rather than its magnitude, and may allow the residual connection in $p_\psi$ to be removed, which could reduce the risk of representational collapse.

\paragraph{Multi-token belief states.} Attention must be introduced into $p_\psi$ with care: if $p_\psi$ attends over the entire history of latents, it can look up past information directly, removing the pressure to compress history into a fixed-size state (\autoref{theorem:nhs_belief_states}). Instead, we can define the belief state as a window of the last $w$ hidden states, $\mathbf{b}_t = (\hidden_{t-w+1}, \dots, \hidden_t)$ (\autoref{fig:multi_token_belief_state}). This state (i) has fixed size, (ii) contains $\hidden_t$, which the output head decodes into the next token, and (iii) admits a transition: $p_\psi$ only needs to predict $\hat{\hidden}_{t+1} = p_\psi(\mathbf{b}_t, X_{t+1})$, since the remaining components of $\mathbf{b}_{t+1} = (\hidden_{t-w+2}, \dots, \hidden_t, \hidden_{t+1})$ can be copied from $\mathbf{b}_t$ by the identity function. Next-token and transition consistency therefore hold over $\mathbf{b}_t$, and \autoref{theorem:nhs_belief_states} applies with $\hidden_t$ replaced by $\mathbf{b}_t$. In practice, $p_\psi$ can be parameterized with the multi-token prediction module of DeepSeek-V3 \citep{liu2024deepseek}, which fuses each hidden state with the embedding of the following token using a linear layer before processing the fused sequence with a transformer block. Here, however, we replace its full attention with sliding-window attention of window size $w$. 

\begin{figure}[htbp]
  \centering
  \includegraphics[width=0.45\linewidth]{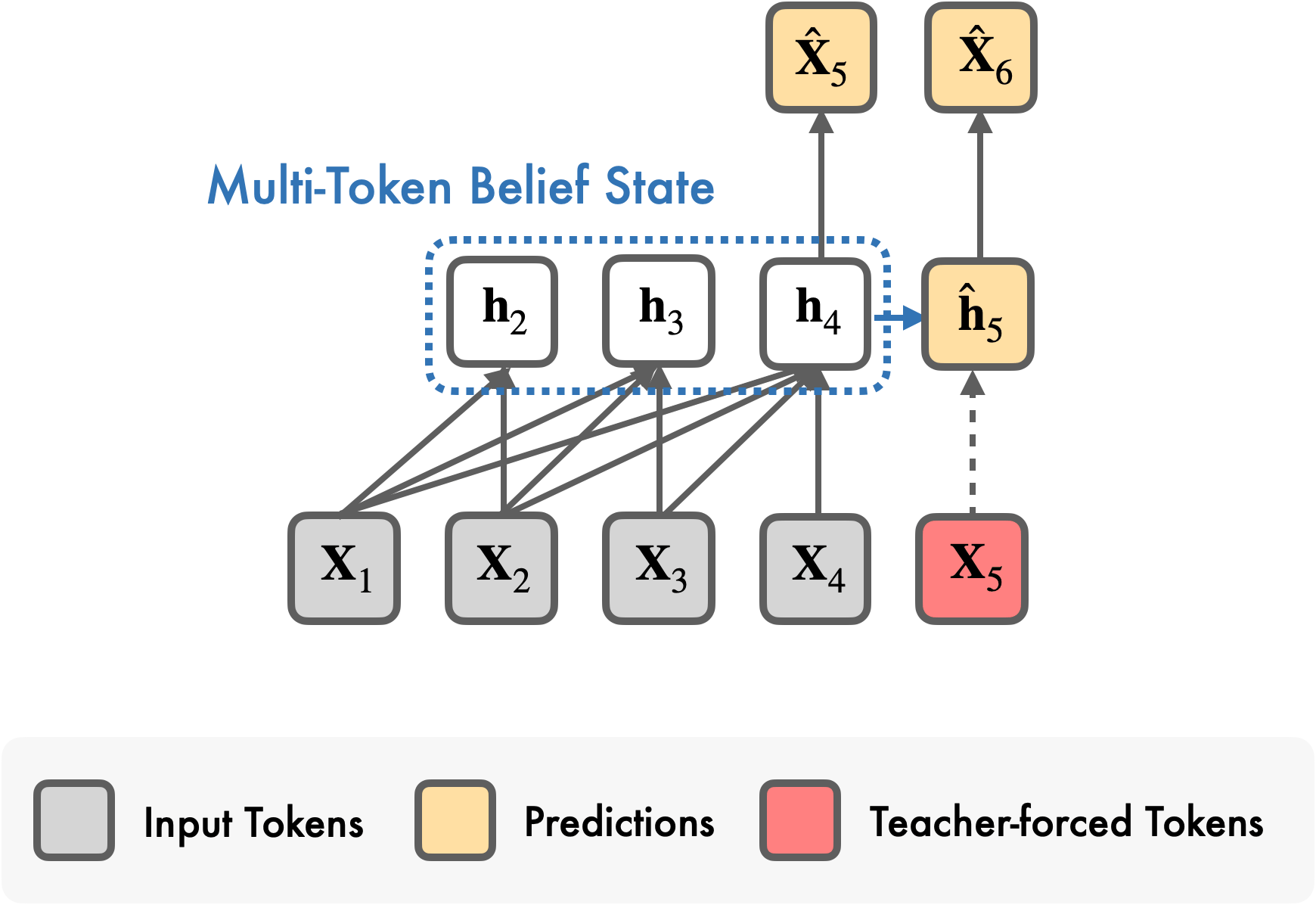}
  \caption{A multi-token belief state with window size $w=3$. The latent dynamics model predicts $\hat{\hidden}_{5}$ from the last three hidden states and the next token $X_5$.}
  \label{fig:multi_token_belief_state}
  \vspace{-1em}
\end{figure}

\paragraph{Training.} Following \citet{li2024eagle1}, we find that adding small uniform noise to the current latent $\hidden_t$ during training makes latent rollouts more robust to compounding errors and improves speculative decoding acceptance. Since the output head is detached from the KL objective, the KL gradient can also be obtained with the KL gradient–recovery trick \citep{teoh2026multitokenpredictionright}, at the cost of a cheaper cross-entropy backward pass, making it compatible with fused cross-entropy kernels.

\section{Conclusion}
In this paper, we introduced Next-Latent Prediction (NextLat), a simple yet powerful framework that augments next-token training with self-supervised latent-space prediction, enabling transformers to learn belief-state representations. Theoretically, we showed that the pressure to form concise latent summaries of past information sufficient to predict the future arises naturally from the NextLat objective. Empirically, NextLat yields more compact, predictive, and generalizable representations across tasks in world modeling, reasoning, planning, and language modeling---all without changing the transformer's architecture, parallel training efficiency, or inference procedure. By reintroducing a recurrent inductive bias through self-predictive latent dynamics, NextLat unifies the inherent bias toward compact and temporally consistent representations of recurrent models with the scalability and parallelism of transformers. Our method is broadly applicable for training transformers in autoregressive sequence modeling domains. Looking ahead, we view NextLat as a step toward training objectives that endow autoregressive sequence models with simpler, more compact, and therefore more generalizable representations of complex systems. 

\begin{ack}
We thank Andrea Zanette, Jason Eisner, Jordan T. Ash, Khai Loong Aw, Michael C. Mozer, Pradeep Varakantham, Ying Fan, and the MSR AI Frontiers team for their valuable discussions and support.

\textbf{Funding.} This work was funded by Microsoft Research and conducted while J.T. was an intern at Microsoft Research.

\textbf{Competing interests.} The authors declare no competing interests beyond their current or former employment at Microsoft.
\end{ack}

%% file: content/appendix.tex
\section{Belief States in Sequence Modeling}
\label{app:belief-state}

Recent work has introduced variants of sequence modeling architectures based on the principle of learning belief states, i.e., BST and JTP. We review these methods here.
Let $\theta$ denote the parameters of a transformer-based model. Let $\hidden_{s:t}$ denote the hidden states produced by the transformer encoder for a token sequence $X_{s:t}$, where $s \leq t$. When we use the notation $\hidden_t$, it is shorthand for $\hidden_{1:t}$. The model's output head produces a categorical distribution over the token vocabulary conditioned on some hidden state representation, i.e., $p_\theta(\cdot \mid \hidden_{s:t})$.

\paragraph{Belief State Transformer.} The Belief State Transformer \citep[BST;][]{hu2025the} learns compact belief states by jointly training a next-token predictor and a previous-token predictor across all possible prefix–suffix decompositions of a sequence, including cases where either the prefix or suffix is empty. Concretely, given a prefix $X_{1:t}$ and a suffix $X_{t+k:T}$ with $k\geq 1$, BST aims to minimize the cross-entropy loss
\begin{align}
\mathcal{L}_\text{BST}(\theta) =  \mathbb{E}_{t < T}\Big[-\log \underbrace{p_\theta(X_{t+1} \mid \hidden_{1:t}, \hidden_{t+k:T})}_{\text{next-token prediction}} - \log\underbrace{p_\theta(X_{t+k-1} \mid \hidden_{1:t}, \hidden_{t+k:T})}_{\text{previous-token prediction}} \Big], 
\label{eq:bst_objective}
\end{align}
where $\hidden_{1:t}$ and $\hidden_{t+k:T}$ are produced by separate transformers. This bidirectional training shapes the hidden representations of the BST into belief states.

\paragraph{Joint Multi-Token Prediction.} Joint multi-token prediction \citep[JTP;][]{ahn2025jtp} aims to learn the joint distribution over the next $d+1$ tokens conditioned on $\hidden_{t}$, where $d$ is the multi-step prediction horizon beyond the next token. Specifically, JTP minimizes the loss
\begin{align}
\mathcal{L}_\text{JTP}(\theta;d, \lambda_\text{MTP})  = \mathbb{E}_{t < T} \Big[-\log \underbrace{p_\theta(X_{t+1} \mid \hidden_{t})}_{\text{next-token prediction}} - \lambda_\text{MTP} \cdot \frac{1}{d}\sum_{i=1}^d\log \underbrace{p_\theta(X_{t+i+1} \mid \fetchhead(\hidden_{t}, X_{t+1:t+i}))}_{\text{joint multi-token prediction}} \Big], 
\label{eq:jtp}
\end{align}
where an additional module $\fetchhead(\hidden_t, X_{t+1:t+i})$ is used to create an embedding combining the teacher-forced tokens $X_{t+1:t+i}$ with $\hidden_t$ and $\lambda_\text{MTP} > 0$ balances next- and multi-token prediction losses. Although \citet{ahn2025jtp} suggest that their method learns ``short-horizon belief states'', they do not formally define the conditions under which this occurs. To understand how JTP learns belief states, we start by defining a \emph{$k$-observable} system.
\begin{definition}[$k$-observability for sequences] \label{def:k_observability}
    A system is $k$-observable if for any two sequences $H =X_{1:t}$ and $H' =X_{1:j}$ that induce the same joint distribution over the next-$k$ tokens, i.e., $\mathbb{P}(X_{t+1:t+k} \mid H) = \mathbb{P}(X_{t+1:t+k} \mid H')$, it follows that their full-horizon conditionals match:
    \begin{align}
    \mathbb{P}(X_{t+1:T} \mid H) = \mathbb{P}(X_{t+1:T} \mid H').
    \end{align}
    In other words, the distribution of all future observations is determined by the distribution of the next $k$ observations. 
\end{definition}
\begin{proposition}[JTP forms belief states in $k$-observable systems] \label{prop:k_belief_states}
Assume the system is $k$-observable and let $k=d+1$. If the joint multi-token prediction model recovers the true joint next-$k$ conditional, i.e., $p_\theta(X_{t+1} \mid \hidden_t)p_\theta(X_{t+2} \mid \hidden_t, X_{t+1})\dots p_\theta(X_{t+k} \mid \hidden_t, X_{t+1:t+k-1}) =\mathbb{P}(X_{t+1:t+k} \mid X_{1:t})$ a.s. for all $t$, then $\hidden_{t}$ is a belief state for $X_{1:t}$.
\end{proposition}
\begin{proof}
By $k$-observability (\cref{def:k_observability}), there exists a measurable map $G$ taking the joint next-$k$ conditional distribution to the full-horizon conditional:
\begin{align*}
    \mathbb{P}(X_{t+1:T} \mid X_{1:t}) = G \Big(\mathbb{P}(X_{t+1:t+k} \mid X_{1:t}) \Big).
\end{align*}
By the premise that JTP recovers the true next-$k$ joint, conditioning on $\hidden_t$ is equivalent to conditioning on $X_{1:t}$ for all bounded measurable functionals of the future. Hence $\hidden_t$ is a belief state for all $t<T$.
\end{proof}
Intuitively, if all possible futures can be distinguished by the next $k$ tokens, then a JTP model that accurately predicts the joint next-$k$ distribution would encode all information necessary to distinguish future trajectories. Note that neither next-token prediction nor multi-token prediction guarantees belief state representations (see \citet{hu2025the}). 

\section{Formal Proof of \autoref{theorem:nhs_belief_states}}
 \label{pf:nhs_belief_states}

The proof follows the intuition illustrated below.  
Optimizing for next-token and transition consistency (\cref{eq:emission_correctness,eq:transition_correctness}) ensures the existence of measurable maps $p_\theta$ and $p_\psi$ that allow recursive decoding of future tokens:
\begin{align*}
    \hidden_t 
    &\xrightarrow[\text{decode token}]{p_\theta} X_{t+1} 
    \xrightarrow[\text{update state}]{p_\psi} \hidden_{t+1} 
    \xrightarrow[\text{decode token}]{p_\theta} X_{t+2} 
    \xrightarrow[\text{update state}]{p_\psi} \hidden_{t+2} 
    \;\cdots\;
    \xrightarrow[]{p_\theta} X_T.
\end{align*}
\begin{proof}
A formal proof proceeds by backward induction on $t$. For the base case  $t = T-1$, the claim follows directly from \cref{eq:emission_correctness}, since $\hidden_{T-1}$ suffices to predict $X_T$.

Now assume $\hidden_{k+1}$ is a belief state for $X_{1:k+1}$. By \cref{def:belief_states}, this implies that $X_{k+2:T}$ is conditionally independent of $X_{1:k+1}$ given $\hidden_{k+1}$.
We will show that $\hidden_k$ is also a belief state for $X_{1:k}$.
From $\hidden_k$, one can generate
\begin{align*}
    X_{k+1} &\sim p_\theta(\cdot \mid \hidden_k), \\
    \hidden_{k+1} &\sim p_\psi(\cdot \mid \hidden_k, X_{k+1}).
\end{align*} 
By next-token and transition consistency (\cref{eq:emission_correctness,eq:transition_correctness}), we have
\begin{align*}
\mathbb{P}(X_{k+1:T}\mid \hidden_k)
&= \mathbb{P}(X_{k+2:T}\mid X_{k+1}, \hidden_k)\,
   \underbrace{\mathbb{P}(X_{k+1}\mid \hidden_k)}_{\text{\cref{eq:emission_correctness}}}
\\
&= \mathbb{P}(X_{k+2:T}\mid X_{k+1}, \hidden_k)\,
   \mathbb{P}(X_{k+1}\mid X_{1:k})
\\
&= \biggl[\int \mathbb{P}(X_{k+2:T},~ \hidden_{k+1}\mid X_{k+1}, \hidden_k) \,d\hidden_{k+1}\biggr]\,
   \mathbb{P}(X_{k+1}\mid X_{1:k})\\
&= \biggl[\int 
      \underbrace{\mathbb{P}(X_{k+2:T}\mid \hidden_{k+1}, X_{k+1}, \hidden_k)}_{\text{Induction hypothesis}} 
      \underbrace{\mathbb{P}(\hidden_{k+1}\mid X_{k+1}, \hidden_k)}_{\text{\cref{eq:transition_correctness}}}
      \, d\hidden_{k+1} \biggr]\,
   \mathbb{P}(X_{k+1}\mid X_{1:k})
\\
&= \biggl[\int 
      \mathbb{P}(X_{k+2:T}\mid \hidden_{k+1})\,
      \mathbb{P}(\hidden_{k+1}\mid X_{1:k+1})
      \, d\hidden_{k+1} \biggr]\,
   \mathbb{P}(X_{k+1}\mid X_{1:k})
\\
&=  \mathbb{P}(X_{k+2:T}\mid X_{1:k+1}) \,
   \mathbb{P}(X_{k+1}\mid X_{1:k}) = \mathbb{P}(X_{k+1:T}\mid X_{1:k}).
\end{align*}
This proves that $\hidden_k$ is also a belief state.
\end{proof}

\section{More Details on NextLat Implementation}
\label{section:more_details_nextlat}
\begin{wrapfigure}{r}{0.25\textwidth}
    \centering
    \vspace{-1em}
    \includegraphics[width=0.65\linewidth]{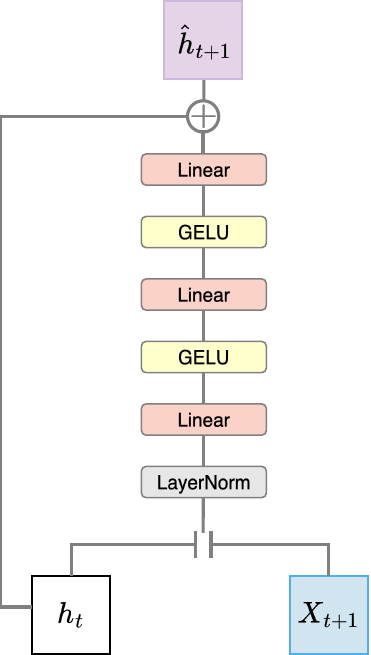}
    \caption{Illustration of the latent transition model $p_\psi$.}
    \label{fig:latent_transition_model}
\end{wrapfigure}
We parameterize the latent transition model $p_\psi$ with a three-layer MLP using GELU \citep{hendrycks2023gaussianerrorlinearunits} activations. The latent transition model takes as input the layer-normalized \citep{ba2016layernormalization} concatenation of the current hidden state $\hidden_t$ and next-token embedding $X_{t+1}$, and outputs a delta update applied via residual connection:
\begin{align}
    \hat{\hidden}_{t+1} = p_\psi(\hidden_t, X_{t+1})= f_\psi(\hidden_t, X_{t+1}) + \hidden_t
\end{align}
where $f_\psi(\cdot)$ predicts the modification to $\hidden_t$ (see \cref{fig:latent_transition_model}). This paper aims to demonstrate that NextLat yields significant performance gains even with a \textbf{simple MLP latent transition}. We foresee even better performances with more sophisticated latent transition model architectures, but we leave that exploration to future work. Because we stop-gradient the output head for the KL objective, we can use the KL gradient--recovery trick \citep{teoh2026multitokenpredictionright} to yield the KL gradient from a cross-entropy backward pass, without the full-vocabulary softmax that explicit KL minimization would require.

Following standard training, we mask token-level losses (i.e., $\mathcal{L}_\text{next-token}$ and $\mathcal{L}_\mathrm{KL}$) corresponding to context or prompt tokens. However, we do not apply masking for $\mathcal{L}_\text{next-h}$ on context tokens, ensuring that belief state representations develop even during context processing. When using sequence packing (e.g., for Manhattan, TinyStories, and FineWeb-Edu), we mask $\mathcal{L}_\mathrm{KL}$ and $\mathcal{L}_\text{next-h}$ terms that cross document boundaries.

\begin{algorithm}[htbp]
    \caption{Pseudocode for NextLat's loss function in PyTorch syntax. The loss functions \inlinecode{python}{cross_entropy_loss()}, \inlinecode{python}{smooth_l1_loss()}, and \inlinecode{python}{KL_loss()} are implemented externally.}
\label{alg:pytorch-code}
\footnotesize
\begin{lstlisting}
import torch.nn as nn 
from copy import deepcopy
    
def loss(batch, targets):
    # batch:(B,T), targets: (B,T)
    # embedding dimension = D, vocabulary size = V
    batch_size, seq_len = batch.shape

    # Encode sequences into latent states
    hidden_states = Transformer(batch) # (B,T,D)

    # Compute next-token loss
    logits_post = Output_Head(hidden_states) # (B,T,V)
    loss_next = cross_entropy_loss(logits_post, targets)

    next_tokens = batch
    next_states = hidden_states 
    current_states = hidden_states 
    loss_next_h = 0
    loss_kl = 0

    # Recursive multi-step predictions, default multi_step_horizon=1;
    # only 1-step prediction
    for _ in range(multi_step_horizon):
        # Shift hidden states back by 1
        current_states = current_states[:, :-1]
        # Shift next tokens and hidden state targets forward by 1
        next_tokens = next_tokens[:, 1:]
        next_states = next_states[:, 1:]
        logits_post = logits_post[:, 1:]
        # Predict next hidden state using latent dynamics model
        # i.e., (h_t, x_{t+1}) -> \hat{h}_{t+1}
        pred_next_states = Latent_Dynamics(current_states, next_tokens) #(B,T,D)
        # Compute next-hidden loss using detached next states as targets
        # i.e., ||h_{t+1} - \hat{h}_{t+1}||^2
        loss_next_h += smooth_l1_loss(pred_next_states, next_states.detach())
    
        # Compute KL loss using detached output head
        # i.e., KL[p(h_{t+1}) || p(\hat{h}_{t+1})]
        logits_prior = deepcopy(Output_Head)(pred_next_states) # (B,T,V)
        loss_kl += KL_loss(logits_post.detach(), logits_prior) # detach posterior

        current_states = pred_next_states

    loss_next_h = loss_next_h / multi_step_horizon
    loss_kl = loss_kl / multi_step_horizon
    # overall NextLat loss
    return loss_next + next_h_lambda * loss_next_h + kl_lambda * loss_kl
\end{lstlisting}
\end{algorithm}

\section{Ablation Studies}
\label{app:ablations}

\begin{wraptable}{r}{0.4\textwidth}
\centering
\vspace{-1.2em}
\begin{tabular}{lc}
\toprule
\textbf{Model} & \textbf{Accuracy (\%)} \\
\midrule
GPT & 33.1 \\
JTP & 39.0 \\
MTP & 39.2 \\
BST & 42.3 \\
\midrule
NextLat ($\mathcal{L}_\mathrm{KL}$ only) & 38.2 \\
NextLat ($\mathcal{L}_\text{next-h}$ only) & \underline{50.7} \\
NextLat ($\mathcal{L}_\mathrm{KL}$ + $\mathcal{L}_\text{next-h}$) & \textbf{54.8} \\
\bottomrule
\end{tabular}
\caption{Accuracy of $d=1$ models on Countdown. Best result is \textbf{bolded}, and second best is \underline{underlined}.}
\label{tab:countdown_ablate_kl_vs_mse}
\vspace{-1em}
\end{wraptable}

In this section, we ablate the key design choices of NextLat. Specifically, we study the effects of the KL ($\mathcal{L}_\mathrm{KL}$) and Smooth L1 ($\mathcal{L}_\text{next-h}$) losses, as well as the use of stop-gradients on the target ($\stopgrad[\hidden_{t+i}]$) in the Smooth L1 loss (\cref{eq:loss_next_hidden}) and on the output head ($p_\theta^{\stopgrad}(\cdot)$) in the KL loss (\cref{eq:loss_kl}). We focus our investigations primarily on Countdown and TinyStories, analyzing how these design choices affect the predictive quality of learned hidden states under linear probing and validation loss behavior. 

We first isolate the effects of the KL and Smooth L1 losses in NextLat on Countdown with $d=1$. Note that only $\mathcal{L}_\text{next-h}$ is required for belief-state convergence, while $\mathcal{L}_\mathrm{KL}$ serves as a complementary supervision signal. In \cref{tab:countdown_ablate_kl_vs_mse}, $\mathcal{L}_\text{next-h}$ proves to be the primary driver behind NextLat's gains---using the Smooth L1 loss alone \textbf{improves accuracy by 17.6 percentage points over GPT (50.7\% vs.\ 33.1\%)} and already outperforms all baselines. In contrast, using the KL loss alone yields only 38.2\%. Combining both losses (i.e., the full NextLat design) achieves the highest accuracy of 54.8\%, indicating that $\mathcal{L}_\mathrm{KL}$ provides a complementary boost on top of $\mathcal{L}_\text{next-h}$.

\begin{figure}[htbp]
     \centering
     \begin{subfigure}[b]{0.46\textwidth}
         \centering
         \includegraphics[width=1.0\textwidth, keepaspectratio]{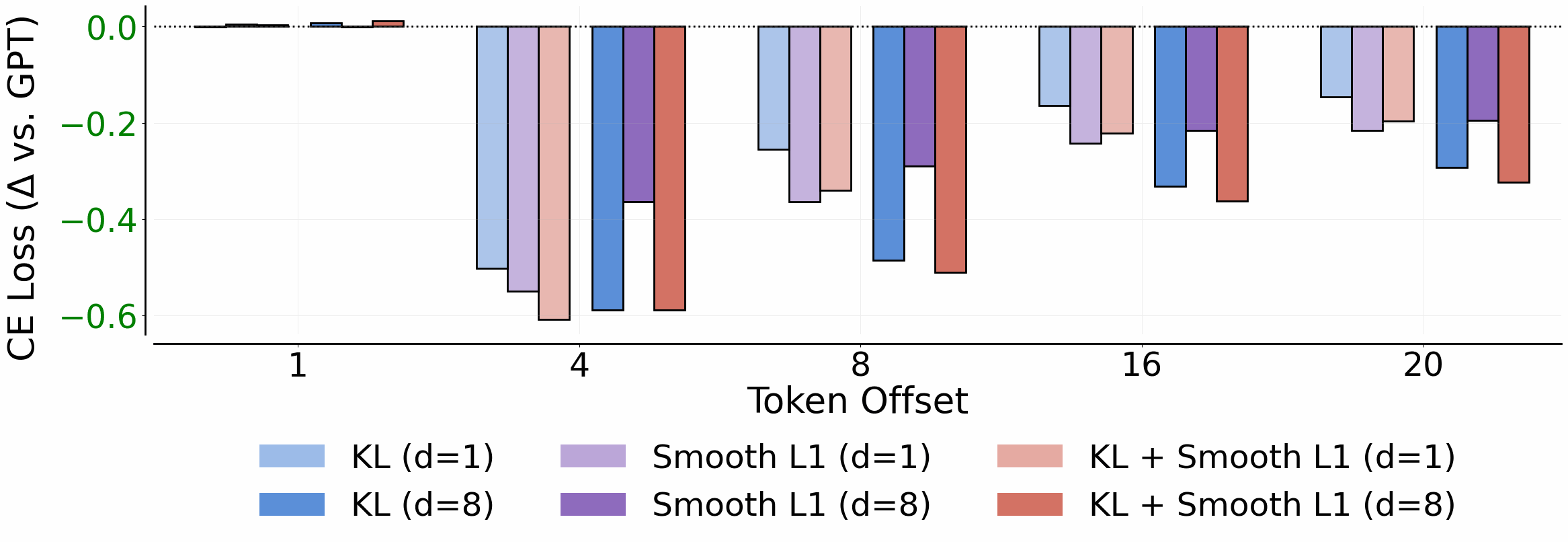}
         \caption{}
         \label{subfig:tinystories_probe_ablate_kl_vs_mse}
     \end{subfigure}
     \hfill
     \begin{subfigure}[b]{0.48\textwidth}
         \centering
         \includegraphics[width=1.0\textwidth, keepaspectratio]{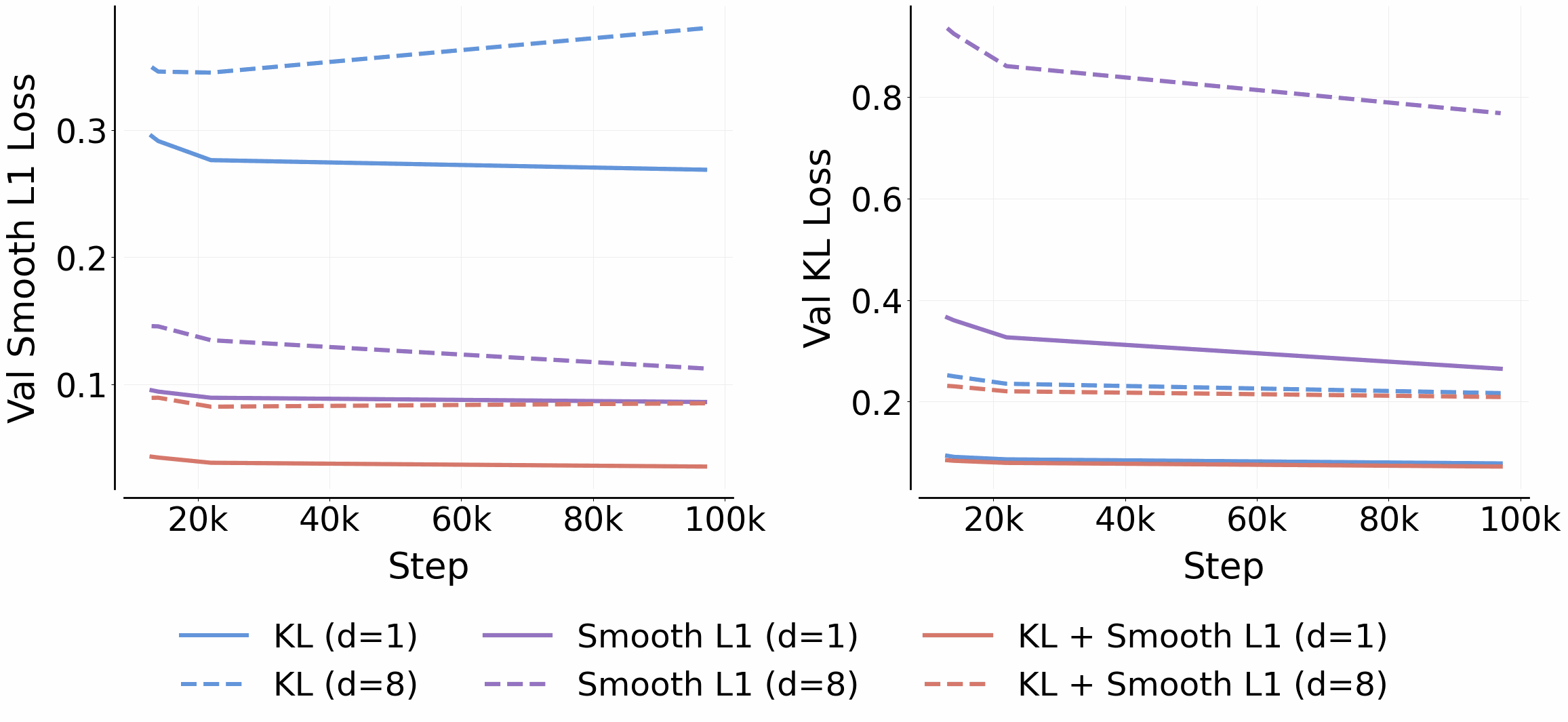}
         \caption{}
         \label{subfig:tinystories_ablate_kl_vs_mse_val_losses}
     \end{subfigure}
     \caption{\textbf{KL and Smooth L1 Loss Ablations:} (a) TinyStories probe performance under ablations of the KL and Smooth L1 losses. (b) TinyStories validation Smooth L1 and KL losses.}
     \label{fig:tinystories_ablate_kl_vs_mse}
\end{figure} 

Next, we ablate the KL and Smooth L1 losses on TinyStories. In \cref{subfig:tinystories_probe_ablate_kl_vs_mse}, we observe that at $d=1$, using Smooth L1 loss alone achieves the strongest probe performance 20 tokens ahead relative to GPT on TinyStories. At $d=8$, however, the combined KL + Smooth L1 objective (i.e., the full NextLat design) performs best, suggesting that the Smooth L1 loss is critical for learning long-range predictive representations, while the KL loss becomes increasingly beneficial at larger horizons. Furthermore, \cref{subfig:tinystories_ablate_kl_vs_mse_val_losses} shows that the combined KL + Smooth L1 objective achieves lower validation KL and Smooth L1 losses than optimizing either objective alone. This confirms that the two losses are complementary and provide mutually beneficial supervision.

\begin{figure}[htbp]
     \centering
     \begin{subfigure}[b]{0.46\textwidth}
         \centering
         \includegraphics[width=1.0\textwidth, keepaspectratio]{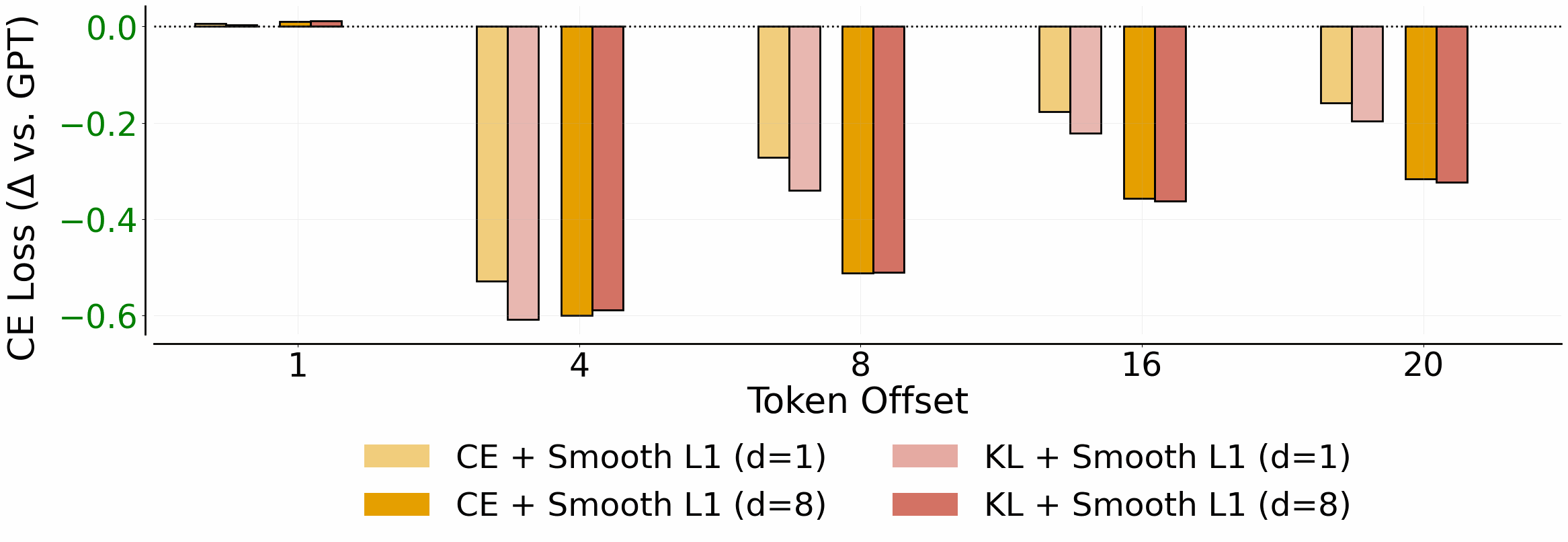}
         \caption{}
         \label{subfig:tinystories_probe_ablate_kl_vs_ce}
     \end{subfigure}
     \hfill
     \begin{subfigure}[b]{0.48\textwidth}
         \centering
         \includegraphics[width=1.0\textwidth, keepaspectratio]{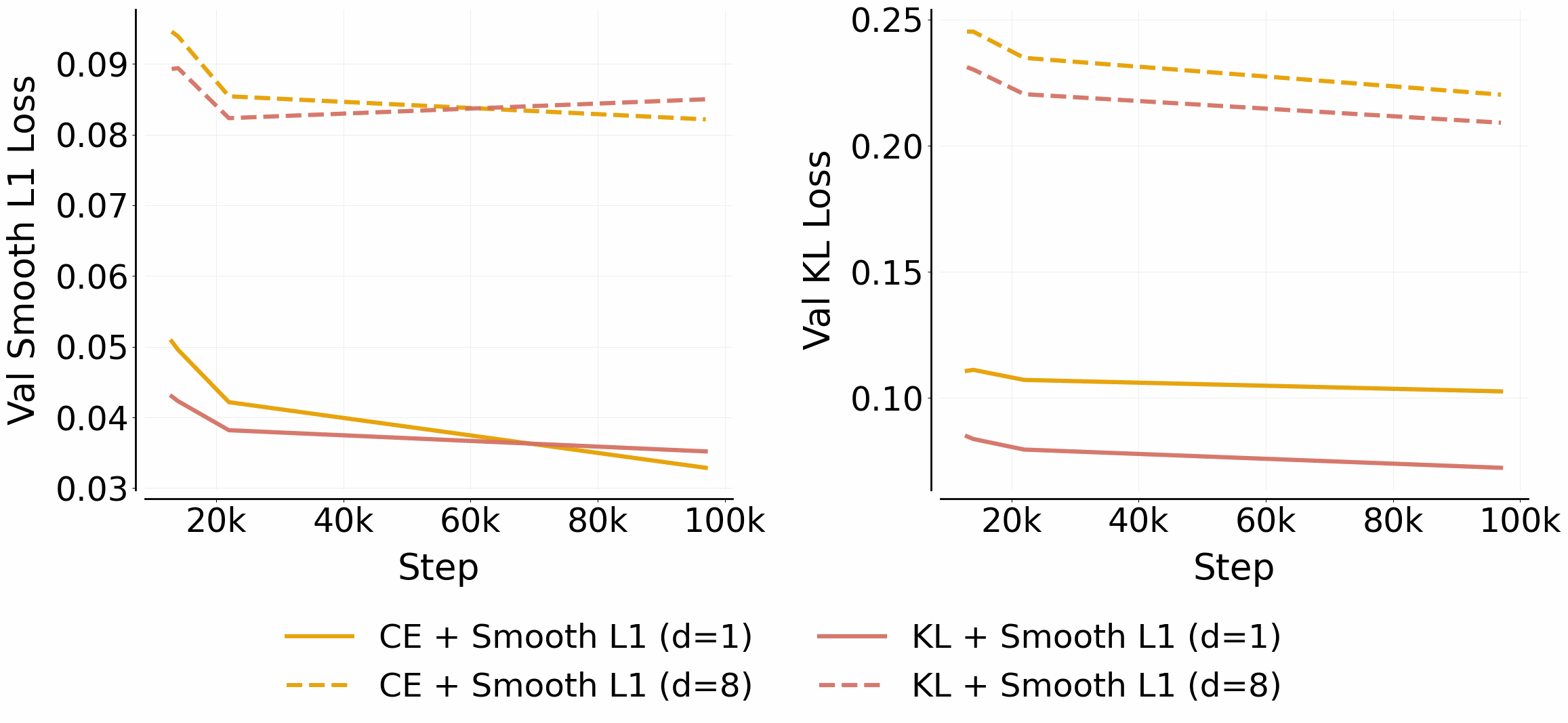}
         \caption{}
         \label{subfig:tinystories_ablate_kl_vs_ce_val_losses}
     \end{subfigure}
     \caption{\textbf{KL and CE Loss Ablations:} (a) TinyStories probe performance when replacing the KL loss in NextLat with a CE loss. (b) TinyStories validation Smooth L1 and KL losses.}
     \label{fig:tinystories_ablate_kl_vs_ce}
\end{figure} 

Next, we compare the KL loss against cross-entropy (CE) as the complementary token-space supervision signal on TinyStories. Note that the CE loss corresponds exactly to the multi-token prediction objective used in JTP and MTP. Unlike the KL loss, which performs self-distillation by matching the predicted token distribution to the model’s own softmax distribution under the target hidden state (see \cref{eq:loss_kl}), the CE loss directly supervises the predicted latent state using the ground-truth next token. As shown in \cref{subfig:tinystories_probe_ablate_kl_vs_ce}, the original NextLat objective (KL + Smooth L1) achieves substantially better probe performance than CE + Smooth L1 at $d=1$, though the gap becomes marginal at $d=8$. \cref{subfig:tinystories_ablate_kl_vs_ce_val_losses} shows that across both $d=1$ and $d=8$, KL + Smooth L1 achieves lower validation KL loss, but slightly worse Smooth L1 loss, than CE + Smooth L1. Overall, these results suggest that at shallow multi-step prediction horizons, the dense distribution-level supervision provided by KL matching induces more predictive representations than CE. However, this advantage appears to diminish at larger horizons. More systematic studies are needed, especially at larger model scales, to determine whether KL or CE is preferable as a complementary supervision signal. For the scope of this work, we stick to using KL loss in all experiments.

\begin{figure}[htbp]
     \centering
     \begin{subfigure}[b]{0.46\textwidth}
         \centering
         \includegraphics[width=1.0\textwidth, keepaspectratio]{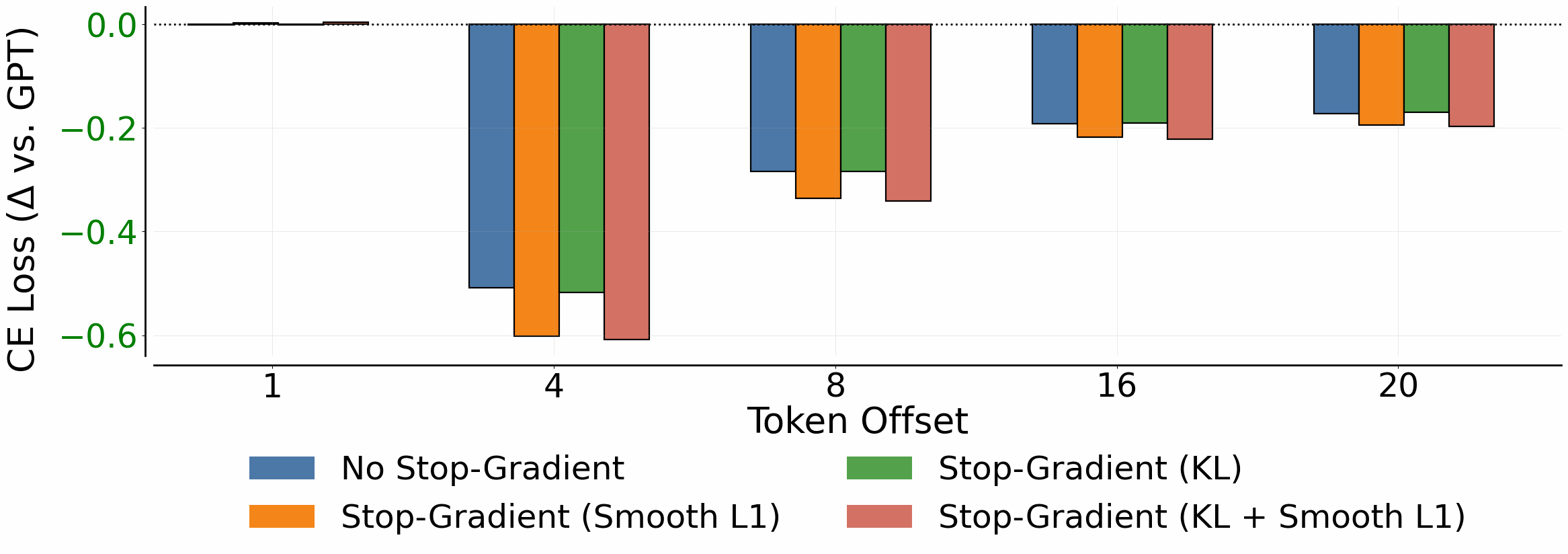}
         \caption{}
         \label{subfig:tinystories_probe_ablate_stop_gradient}
     \end{subfigure}
     \hfill
     \begin{subfigure}[b]{0.48\textwidth}
         \centering
         \includegraphics[width=1.0\textwidth, keepaspectratio]{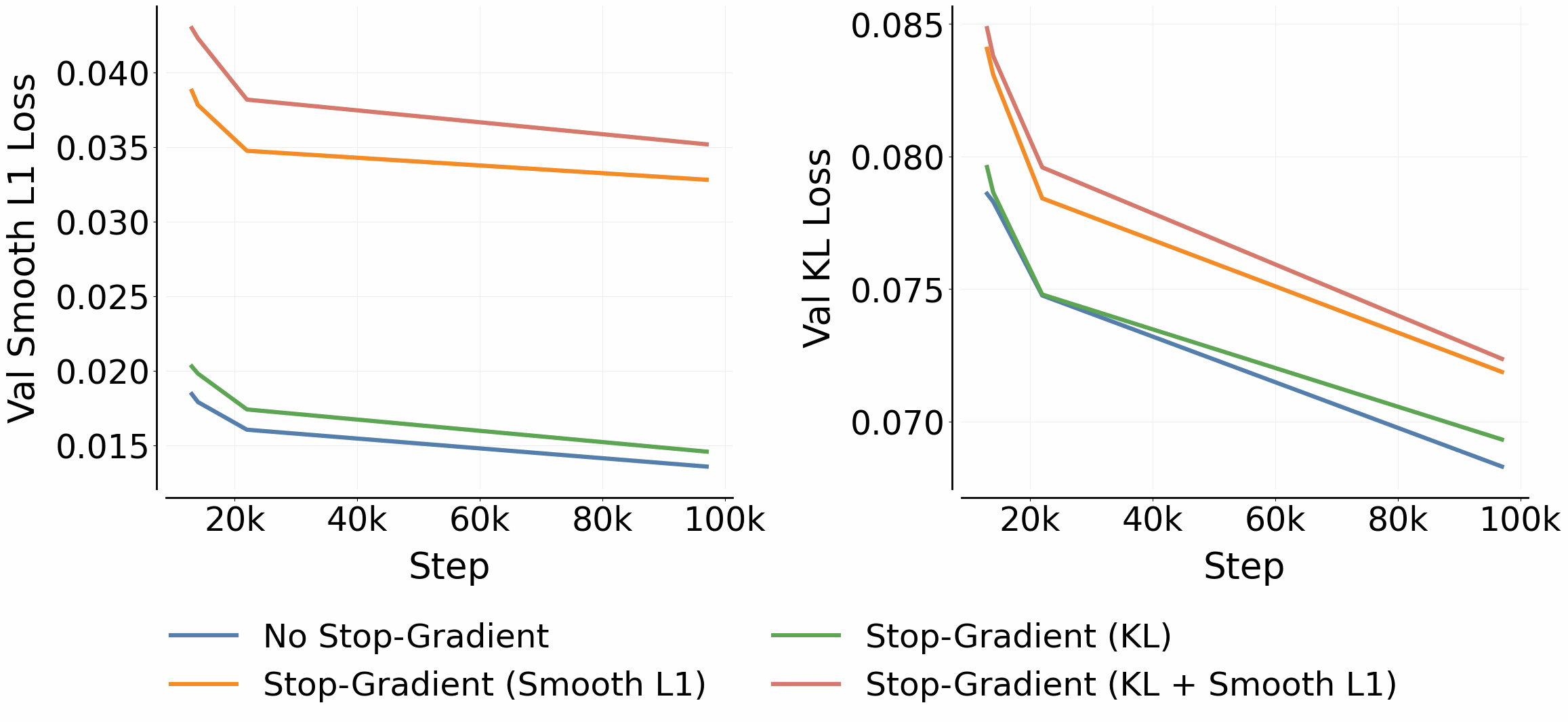}
         \caption{}
         \label{subfig:tinystories_ablate_stop_gradient}
     \end{subfigure}
     \caption{\textbf{Stop-Gradient Ablations:} (a) TinyStories probe performance under different stop-gradient configurations. (b) TinyStories validation Smooth L1 and KL losses. Note that all models here use both KL and Smooth L1 losses and differ only in where stop-gradients are applied in those loss components. ``Stop-Gradient (Smooth L1)'' applies stop-gradient only to the target hidden state in the Smooth L1 loss, ``Stop-Gradient (KL)'' applies stop-gradient only to the output head in the KL loss, and ``Stop-Gradient (KL + Smooth L1)'' applies stop-gradients to both components. All models use $d=1$ here.}
     \label{fig:tinystories_ablate_stop_gradient}
\end{figure} 

Finally, we study the use of stop-gradients in \cref{fig:tinystories_ablate_stop_gradient} with $d=1$ on TinyStories. While \cref{subfig:tinystories_ablate_stop_gradient} shows that removing the stop-gradient on the Smooth L1 loss improves validation Smooth L1 and KL losses slightly ($\sim 0.02$), it does not improve the predictive quality of the learned representations. As shown in \cref{subfig:tinystories_probe_ablate_stop_gradient}, applying stop-gradients to both the target in the Smooth L1 loss and the output head ($p_\theta^{\stopgrad}(\cdot)$) in the KL loss yields the best probing performance across all token offsets. Stop-gradients also improve efficiency by reducing the number of backward passes through the transformer. In particular, the KL loss introduces an additional forward and backward pass through the output head, which can be computationally expensive in language models due to large vocabulary sizes\footnote{Note that multi-token prediction models also incur an additional forward+backward pass through the output head for each extra token prediction.}. Removing this extra backward pass significantly improves both speed and memory efficiency. Empirically, on tasks like Countdown and Path-Star graph, we also observe that the stop-gradient (especially on the Smooth L1 loss) is essential for high accuracy. 

Overall, these ablations suggest that NextLat’s effectiveness arises from the combination of Smooth L1 and KL losses together with carefully placed stop-gradients, all of which contribute to learning predictive belief-state representations in transformers.

\section{NextLat Pretraining Quirks}
\label{section:nextlat_quirks}
\begin{figure}[htbp]
    \centering
    \begin{subfigure}[t]{0.26\textwidth}
        \centering
        \includegraphics[width=\textwidth]{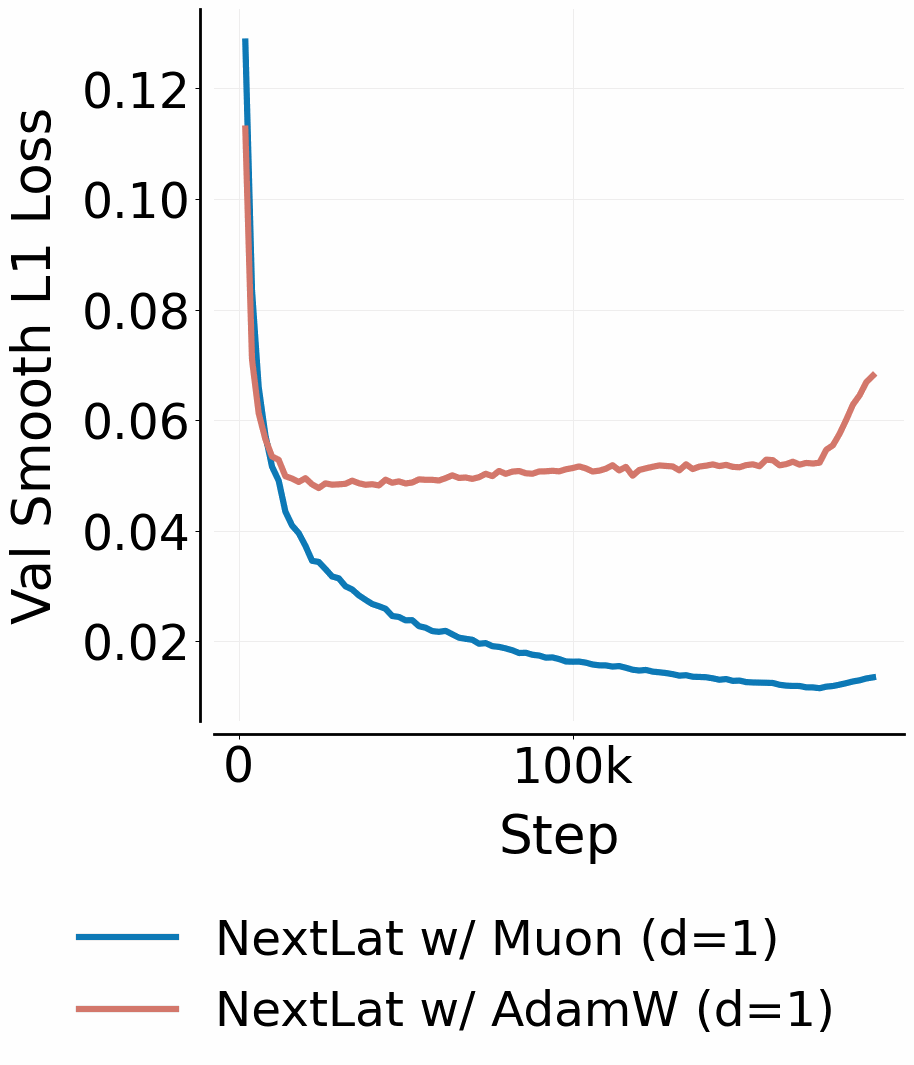}
        \caption{}
        \label{fig:MSE_adamw_vs_muon}
    \end{subfigure}%
    \hfill
    \begin{subfigure}[t]{0.26\textwidth}
        \centering
        \includegraphics[width=\textwidth]{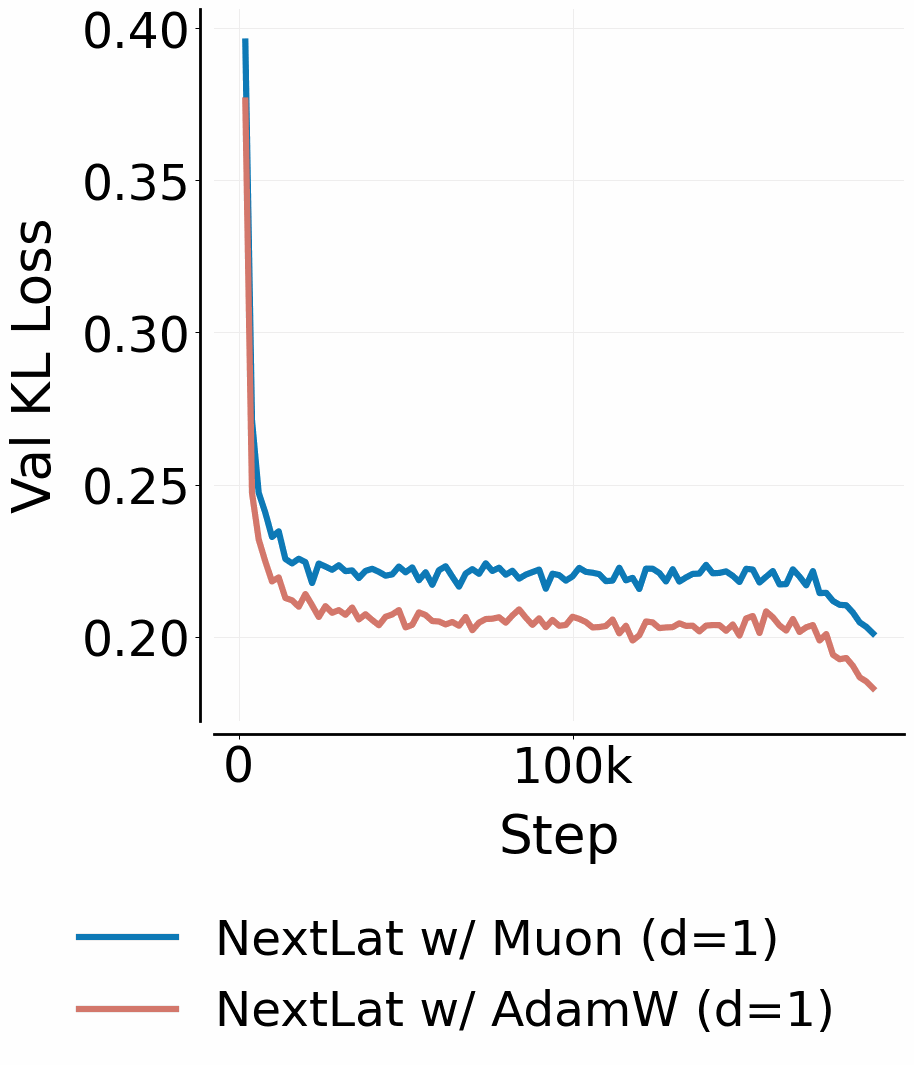}
        \caption{}
        \label{fig:KL_adamw_vs_muon}
    \end{subfigure}
    \hfill
    \begin{subfigure}[t]{0.26\textwidth}
        \centering
        \includegraphics[width=\textwidth]{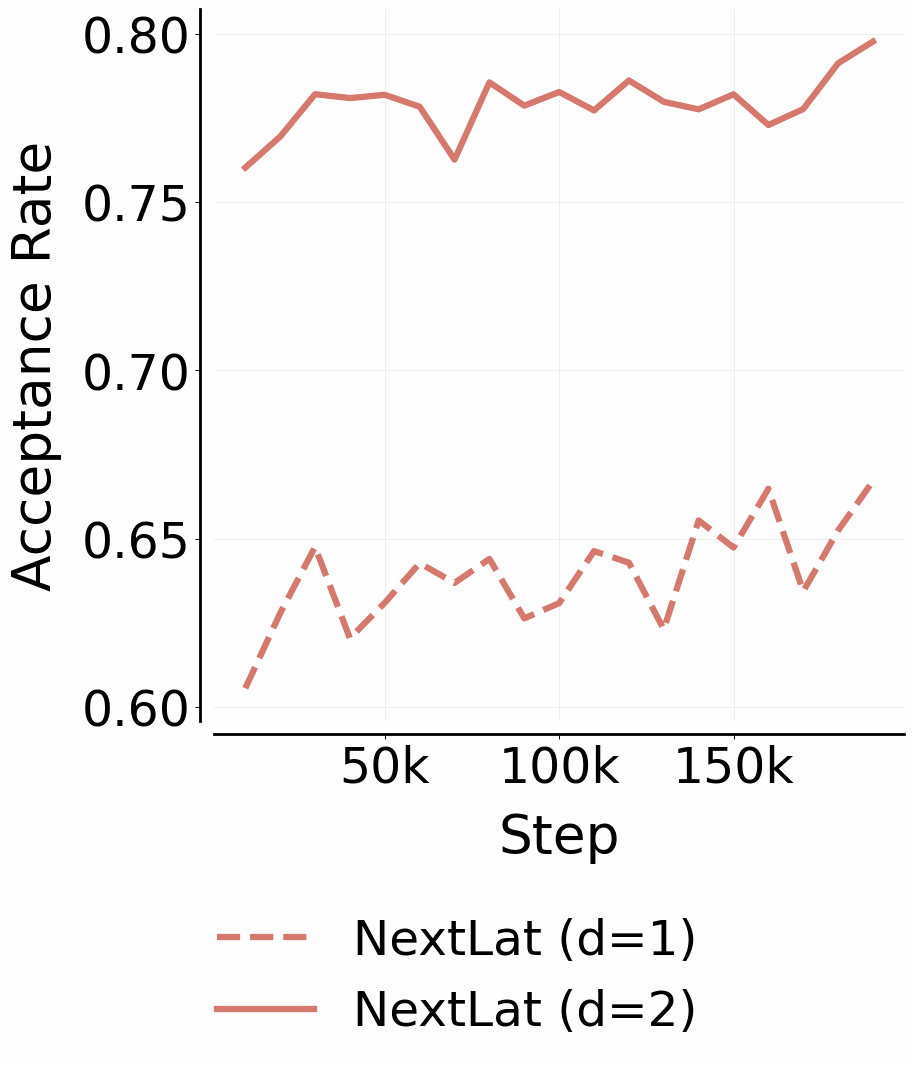}
        \caption{}
        \label{fig:spec_acceptance_rate_progress}
    \end{subfigure}
    \caption{\textbf{FineWeb-Edu Pretraining Curves:} (a) Validation Smooth L1 loss during training. (b) Validation KL loss during training. (c) Speculative decoding acceptance rate over a 5-token lookahead horizon (excluding the next-token prediction).}
    \label{fig:nextlat_quirks}
    \vspace{-1em}
\end{figure}
In this section, we expose some optimization quirks that we faced during pretraining with NextLat. First, \cref{fig:MSE_adamw_vs_muon} shows vastly different Smooth L1 loss trajectories when using Muon \citep{jordan2024muon} versus AdamW \citep{loshchilov2019decoupledweightdecayregularization} optimizer under the same learning rate schedule (using the Muon update scaling rule proposed by \citet{liu2025muon}). One possible explanation is that Muon’s orthogonalized update rule enables more efficient latent representation shaping and next-latent prediction optimization. However, \cref{fig:KL_adamw_vs_muon} shows that AdamW achieves better KL loss optimization than Muon. 

Importantly, both optimizers exhibit \textbf{an increase in Smooth L1 loss near the end of training} corresponding to the learning rate cooldown stage, though the effect is substantially more pronounced for AdamW. The cause of this behavior remains unclear. We explored several potential fixes to no avail, including: 
\begin{enumerate}
    \item replacing Smooth L1 with MSE loss,
    \item removing stop-gradients from $\mathcal{L}_\text{next-h}$,
    \item tuning loss coefficients $\lambda_\text{next-h}$ and $\lambda_\mathrm{KL}$,
    \item disabling learning rate decay for the latent transition model, and
    \item disabling weight decay for the latent transition model.
\end{enumerate}
We leave a deeper investigation of this phenomenon to future work. Reassuringly, the increase in Smooth L1 loss does not appear to degrade the coherence of the latent transition model. As shown in \cref{fig:spec_acceptance_rate_progress}, the speculative decoding acceptance rate over a lookahead horizon of 5 tokens continues to improve throughout training, even during the phase where the Smooth L1 loss increases under AdamW. This suggests that the rise in Smooth L1 loss may reflect changes in the scale or geometry of the latent states, rather than degradation in the token-level coherence of the latent transition dynamics. 

On a side note, we observed gradient norm spikes that led to training instability when training MTP with Muon, which unfortunately limited our pretraining analysis to AdamW (despite Muon showing a more favorable learning trajectory with NextLat).

\section{Extra Experiment Details}
\label{section:experiment_details}

In this section, we provide additional experimental setup details and supplementary results for the experiments in the main body. All experiments were run on NVIDIA RTX A5000, NVIDIA H100 NVL, and NVIDIA B200 GPUs. \cref{tab:exp_hyperparams} summarizes the training, model, and NextLat hyperparameters across all experimental domains.

\begin{table}[htbp]
  \centering
  \resizebox{\columnwidth}{!}{%
  \begin{tabular}{lcccccc}
    \toprule
    & \textbf{Manhattan} & \textbf{Countdown} & \textbf{Path-Star} & \textbf{TinyStories} & \textbf{$\mathbf{A_5}$} & \textbf{FineWeb-Edu} \\
    \midrule
    \multicolumn{6}{l}{\textbf{Training Parameters}} \\
    Steps & 400k (6 epochs) & 100k (204 epochs) & 20k (51 epochs) & 100k (8.5 epochs) & 400k (410 epochs) & 191k (100B tokens) \\
    Batch Size & 256 & 1024 & 512 & 256 & 1024 & 512 \\
    Learning Rate & 1e-4 & 3e-4 & 5e-4 & 3e-4 & 1e-4 & 4e-4 \\
    Learning Rate Schedule & Constant & Constant & Constant & Constant & Constant & WSD \\
    Weight Decay & 0.01 & 0.1 & 0.1 & 0.1 & 0.01 & 0.1 \\
    Clip Gradient Norm & - & 1 & 100 & 1 & 1 & 1 \\
    Optimizer & \multicolumn{6}{c}{AdamW \citep{loshchilov2019decoupledweightdecayregularization} with $\beta_1=0.9, \; \beta_2=0.95$} \\
    \midrule
    \multicolumn{6}{l}{\textbf{Model Parameters}} \\
    Layers & 48 & 12 & 12 & 8 & 2 & 22 \\
    Heads & 8 & 12 & 6 & 8 & 8 & 16 \\
    Hidden Dimension & 384 & 768 & 384 & 768 & 512 & 2048 \\
    \midrule
    \multicolumn{6}{l}{\textbf{NextLat Parameters}} \\
    $\lambda_\text{next-h}$ & 1.0 & 2.0 & 1.0 & 1.0 & 1.0 &1.0 \\
    $\lambda_\mathrm{KL}$ & 0.1 & 1.0 & 1.0 & 1.0 & 0.0 & 1.0 \\
    $p_\psi$ MLP Hidden Dimension & 1536 & 768 & 384 & 2048 & 1024 & 6528 \\
    $p_\psi$ MLP Layers & 3 & 3 & 3 & 3 & 3 & 3 \\
    \bottomrule
  \end{tabular}
  }
  \caption{Training, Model, and NextLat hyperparameters across all benchmarks.}
  \label{tab:exp_hyperparams}
\end{table}

The hyperparameters reported in \cref{tab:exp_hyperparams} were chosen through a small-scale search, exploring $\lambda_\text{next-h} \in \{1.0, 2.0\}$ and $\lambda_\mathrm{KL} \in \{0.1, 1.0\}$, guided by empirical observations and intuition. Encouragingly, we find that NextLat performs robustly over a wide range of settings. In particular, $\lambda_\mathrm{KL}$ requires minimal tuning; $\lambda_\mathrm{KL}=1.0$ works decently across all tasks. After all, it primarily serves as a complementary alignment objective. Likewise, $\lambda_\text{next-h}=1.0$ is effective in most cases, though slightly higher values (e.g., $\lambda_\text{next-h}=2.0$) were beneficial when the next-latent regression loss is of much smaller scale than the token-level losses (i.e., $\mathcal{L}_\text{next-token}$ and $\mathcal{L}_\mathrm{KL}$). 

For MTP and JTP, we sweep multi-token prediction loss weights $\lambda_\text{MTP} \in \{0.1, 0.2, 0.4, 0.6, 0.8, 1.0\}$ to ensure fair baseline comparisons. For FineWeb-Edu pretraining, extensive hyperparameter tuning is computationally prohibitive, so we use uniform loss weights across all methods (i.e., $\lambda_\text{next-h} = \lambda_\mathrm{KL} = \lambda_\text{MTP} = 1.0$) to provide the fairest comparison possible under our compute constraints.

\subsection{Manhattan Taxi Rides}
\label{subsection:manhattan_details}
Here, we provide additional details on our training and evaluation setups for the Manhattan Taxi Rides benchmark and clarify key differences from the original study of \citet{vafa2024evaluating}. 

\paragraph{Training.} Since this task inherently requires state tracking (i.e., tracking position within Manhattan), and increasing model depth is known to benefit transformers on such tasks \citep{merrill2025illusionstatestatespacemodels}, we employ 48-layer transformers with 384 hidden dimensions and 8 attention heads (88M parameters). This differs from \citet{vafa2024evaluating}, which used 12 layers, 768 hidden dimensions, and 12 heads for their smaller transformer variant. We found that increasing model depth yielded substantial performance gains, whereas increasing hidden dimensionality offered negligible improvement. As shown in \cref{tab:manhattan_results}, the effective latent rank of our models is substantially smaller than 384, suggesting that large hidden dimensions are unnecessary. 

Unlike the original study, which trained models for only one epoch, our models are trained for six epochs, as we observed that performances generally do not converge within a single epoch. This also helps rule out potential ``grokking'' phenomena \citep{power2022grokkinggeneralizationoverfittingsmall}, where generalization improves only after extended periods of overfitting. To enable longer training without substantially increasing runtime, we apply sequence packing, i.e., concatenating multiple sequences into longer ones while masking cross-sequence attention. This enables efficient utilization of GPU memory and computation. Most models complete six training epochs in under three days on a single NVIDIA H100 NVL GPU, except MTP, which has substantially more parameters than the others.

\paragraph{Evaluation.} Models are trained on random traversals of length 100 connecting pickup and dropoff intersections. This task inherently requires not only state tracking but also planning, as models must reason over possible future paths to generate valid 100-step trajectories that reach the destination while avoiding dead ends, i.e., road segments disconnected from the goal due to the one-way streets. Random traversals rarely correspond to true shortest paths and do not provide an inductive bias toward learning the shortest path algorithm, which relies on dynamic programming. Consequently, unlike \citet{vafa2024evaluating}, who evaluated pairs with shortest paths of up to 100 steps, we limit evaluation pairs to paths of up to 50 steps. This adjustment ensures that evaluation pairs do not demand long-horizon planning beyond the training distribution, which might otherwise force the model to produce predictions that compensate for planning failure. This also ensures that inconsistencies in a model’s internal map more likely reflect world-model incoherence rather than artifacts of long-horizon planning limitations. These evaluation pairs are used to generate \cref{fig:manhattan_maps_full} and to compute the sequence compression and detour robustness metrics, following the procedure of \citet{vafa2024evaluating}.

For the effective latent rank, we pass a batch of 256 sequences (each of length 256) through the model to obtain the hidden state matrix. Singular values smaller than $1\mathrm{e}{-12}$ are discarded, and the effective rank is then computed following \citet{roy2007effective}. For GPT and NextLat, we use the final-layer hidden states. For JTP, we extract the hidden states immediately before the self-attention module in the Fetch head (see Equations 4–5 in \citet{ahn2025jtp}). For MTP, we use the output of the next-token prediction head to compute the effective rank.

\subsection{Countdown}

We largely follow \citet{gandhi2024stream} for the Countdown training and evaluation setup. Each problem consists of four input numbers and a solution sequence comprising three equations, consistent with prior work \citep{gandhi2024stream,ye2025beyond}. A training example is formatted as
\begin{equation*}
\underbrace{14,83,88,91}_{\text{inputs}},\overbrace{23}^{\text{target}}\:|\underbrace{\:83-14=69,\:91-88=3,\:69/3=23}_{\text{solution}}
\end{equation*}
where the first four numbers are the inputs, the fifth is the target number, and the pipe symbol ``\texttt{|}'' separates the input prompt from the solution. During training, loss values corresponding to the input prompt are masked out.

Previous studies involving the Countdown benchmark used pretrained GPT-2 byte-pair encoding tokenizers, which do not necessarily tokenize multi-digit numbers as single units. In contrast, we construct a custom tokenizer that assigns each integer from 1 to 10,000 to a unique token, ensuring that every number in the sequence is represented atomically. The arithmetic operators and delimiters, i.e., $\{\; |\;,\;+\;,\;-\;,\;\times\;,\;\div\;\}$, are each assigned their own token indices. Due to the large branching factor of the Countdown problem, we insert eight pipe symbols (“\texttt{|}”) between the input and the solution as pause tokens \citep{goyal2023think}, allowing the model additional computation steps to plan before generating its answer.

\subsection{Path-Star}

Our Path-Star data preparation, training, and evaluation follow \citet{bachmann24a}, except that we increase the weight decay to 0.1, which we found helpful for stable convergence and higher solve rates in the multi-step prediction methods (i.e., MTP, JTP, and NextLat). We evaluate each model’s ability to generate the correct arm on 20k held-out test instances. Unlike \citet{hu2025the} and \citet{ahn2025jtp}, who generate a fresh set of graphs every batch, we adopt the original, more challenging setup of \citet{bachmann24a}, which uses a fixed sample size of 200k and node values sampled from $N = 100$. This difference accounts for the performance gap observed in the BST and JTP baselines in \cref{fig:stargraph_result}. The Path-Star experiment is designed to expose the myopic behavior of teacher-forced next-token prediction, which can encourage models to exploit superficial regularities---an effect referred to as the Clever Hans cheat \citep{bachmann24a}. Because the task’s sample space grows exponentially with graph size, identifying the correct algorithm that generalizes across all graph instances is highly nontrivial. While not conclusive, our results suggest that latent-space prediction and the inductive bias toward compressing history into belief states promote better discovery of generalizable solutions in data-constrained settings.

\subsection{TinyStories}
\label{subsection:tinystories_details}
\begin{figure}[htbp]
    \centering
    \includegraphics[width=1.0\linewidth]{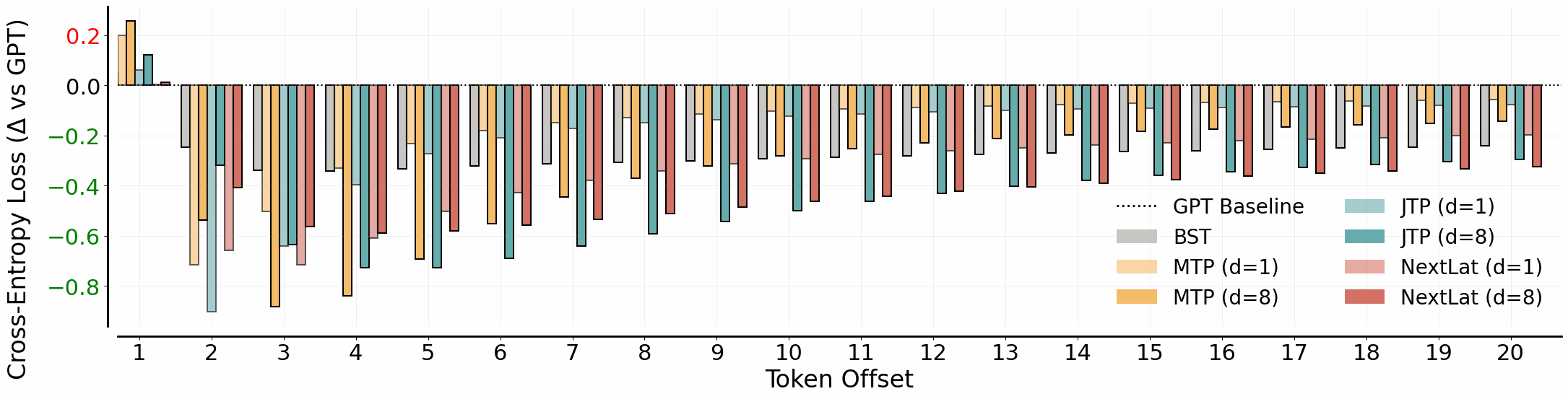}
    \caption{Full plot version of \cref{fig:tinystories_result} that shows probe performance across all 20 token offsets.}
    \label{fig:tinystories_results_full20}
\end{figure}

Our TinyStories setup follows exactly \citet{hu2025the}. After pretraining on TinyStories, we train linear (one-layer) probes on the hidden states of the frozen transformer models for an additional 20k steps on the same dataset. All probe training hyperparameters (e.g., learning rate, batch size) match those used during pretraining (see \cref{tab:exp_hyperparams}).

For GPT, JTP, and NextLat, the choice of hidden states follows the setup used for measuring effective rank in the Manhattan taxi rides task, as described in \autoref{subsection:manhattan_details}. For BST, we use the final-layer hidden states of the forward transformer encoder. For MTP \citep{gloeckle2024}, we use the output of the shared transformer trunk, i.e., the hidden state before it branches into separate transformer heads for multi-token prediction, as this final shared representation contains the most predictive information about future tokens.

\subsection{$\mathbf{A_5}$ State Tracking}
The group $A_5$ consists of the even permutations of 5 elements.  Intuitively, each token in the sequence represents a permutation operation that rearranges the five elements, and after each operation the model must output the resulting arrangement (i.e., state). The task is therefore fundamentally a state-tracking problem: the model must maintain and update an internal representation of the current state as new permutation operations are sequentially composed.

Our training setup follows \citet{merrill2025illusionstatestatespacemodels}.  We train the models on 1 million unique 12-token sequences over the group $A_5$. We then evaluate their length generalization capabilities on approximately 100k 36-token sequences. For NextLat training, we use \textbf{one-step supervision} ($d=1$) and optimize only the regression objective ($\lambda_\text{next-h} = 1$), without token-level supervision ($\lambda_\mathrm{KL} = 0$). This ensures that any observed length generalization of the latent dynamics model (RNN) arises purely from faithful next-latent prediction rather than auxiliary multi-token prediction signals.

We also experimented with transformers without positional embeddings (NoPE), motivated by prior work suggesting improved length generalization \citep{kazemnejad2023the}. However, we found that NoPE degraded performance in this setting, so all experiments use Rotary Position Embeddings \citep[RoPE;][]{su2023roformerenhancedtransformerrotary}. Unlike RNNs, whose hidden states evolve sequentially and therefore implicitly encode token order, transformers process sequences in parallel and rely heavily on positional embeddings to represent token order information during training. Interestingly, despite positional embeddings being explicitly included in the transformer hidden states, the latent dynamics model trained on one-step latent transitions derived from the transformer still generalizes beyond the training sequence length. This suggests that the learned transition dynamics are robust to the transformer's positional encodings.

\subsection{FineWeb-Edu Pretraining}
Following standard practices, all models are optimized with AdamW using a peak learning rate of $4\times10^{-4}$, weight decay of $0.1$, and gradient clipping at $1.0$. We use a global batch size of 0.5M tokens and a sequence length of 1024 tokens. The learning rate follows a Warmup-Stable-Decay (WSD) schedule \citep{hu2024minicpm}, consisting of a 1B-token linear warmup followed by a 10B-token linear decay phase. All models use the GPT-2 tokenizer with a vocabulary size of 50,257. 

We evaluate the zero-shot accuracy of the pretrained models on nine standard language modeling benchmarks: Wikitext \citep[Wiki;][]{merity2016pointer}, LAMBADA \citep[LAMB., standard version;][]{paperno2016lambada}, PIQA \citep{bisk2020piqa}, HellaSwag \citep[HellaS.;][]{zellers2019hellaswag}, WinoGrande \citep[Wino.;][]{sakaguchi2021winogrande}, ARC-easy (ARC-e) and ARC-challenge (ARC-c) \citep{clark2018think}, Social IQa \citep[SIQA;][]{sap2019social}, and SciQ \citep{welbl2017crowdsourcing}. We also evaluate the self-speculative decoding performance of JTP, MTP, and NextLat across Wikipedia \citep{wikidump}, Books \citep[BookCorpusOpen;][]{bookcorpusopencard}, Code \citep[Stack-Edu;][]{allal2025smollm2smolgoesbig}, and Math \citep[OpenWebMath;][]{paster2023openwebmath} domains. To do so, we sample 1024 prompts of length 512 tokens from each dataset and generate 512-token continuations using the speculative sampling algorithm of \citet{leviathan2022fast}. We then measure the speedup in inference of each model using self-speculative decoding compared to naive autoregressive sampling from the transformer, measured on $8\times$NVIDIA B200 GPUs. 

\subsubsection{Results with MTP/JTP ($d=4$)}
\label{subsubsection:d4_results}

\begin{table}[htbp]
\centering
\renewcommand{\arraystretch}{0.9}
\resizebox{0.9\columnwidth}{!}{%
\begin{tabular}{l|ccc|cccccccccc}
\toprule
\textbf{Model} 
& \begin{tabular}{@{}c@{}}\textbf{FW-Edu} \\ ppl $\downarrow$\end{tabular} 
& \begin{tabular}{@{}c@{}}\textbf{Wiki.} \\ ppl $\downarrow$\end{tabular} 
& \begin{tabular}{@{}c@{}}\textbf{LAMB.} \\ ppl $\downarrow$\end{tabular}  
& \begin{tabular}{@{}c@{}}\textbf{LAMB.} \\ acc $\uparrow$\end{tabular} 
& \begin{tabular}{@{}c@{}}\textbf{PIQA} \\ acc $\uparrow$\end{tabular}
& \begin{tabular}{@{}c@{}}\textbf{HellaS.} \\ acc $\uparrow$\end{tabular}
& \begin{tabular}{@{}c@{}}\textbf{Wino.} \\ acc $\uparrow$\end{tabular}
& \begin{tabular}{@{}c@{}}\textbf{ARC-e} \\ acc $\uparrow$\end{tabular}
& \begin{tabular}{@{}c@{}}\textbf{ARC-c} \\ acc $\uparrow$\end{tabular}
& \begin{tabular}{@{}c@{}}\textbf{SIQA} \\ acc $\uparrow$\end{tabular}
& \begin{tabular}{@{}c@{}}\textbf{SciQ} \\ acc $\uparrow$\end{tabular}
& \textbf{Avg.} \\
\midrule

GPT 
& \textbf{10.52} & \textbf{17.93} & 20.26 & 42.07 & 73.45 & \textbf{58.79} & \underline{60.46}
& 68.18 & 39.16 & 42.32 
& 86.10
& 58.82 \\

JTP (d=1) 
& 11.08 & 19.28 & 21.88 & 41.35 & \textbf{74.92} & 57.43 & 58.64
& 68.73 & 39.25 & 42.99 
& 87.30
& \underline{58.83} \\

JTP (d=2) 
& 11.18 & 19.60 & 22.11 & 41.37 & 73.34 & 56.84 & 59.98
& 68.86 & 38.57 & \textbf{43.35}
& 86.70
& 58.63 \\

JTP (d=4) 
& 11.29 & 20.45 & 20.65 & 41.90 & 73.45 & 56.58 & 57.70
& 69.73 & \underline{39.68} & 42.53
& \textbf{88.50}
& 58.76 \\

MTP (d=1) 
& 10.90 & 18.82 & 20.23 & 41.26 & \underline{74.32} & 58.05 & \textbf{60.54}
& 68.52 & 38.91 & 42.84 
& 85.40
& 58.73 \\

MTP (d=2) 
& 11.00 & 18.61 & \underline{18.34} & \underline{43.43} & 72.80 & 57.92 & 59.35
& 68.35 & 39.08 & 41.97 
& 86.60
& 58.69 \\

MTP (d=4) 
& 11.10 & 18.97 & 22.75 & 40.69 & 73.72 & 57.39 & 58.33
& \textbf{70.20} & 39.51 & 42.12 
& 85.90
& 58.48 \\

\midrule
\textbf{NextLat (d=1)} 
& \underline{10.83} & \underline{18.39} & 19.77 & 41.08 & 73.07 & \underline{58.35} & 59.27
& 69.65 & \underline{39.68} & \underline{43.24} 
& 86.00
& 58.79  \\

\textbf{NextLat (d=2)} 
& 10.88 & 18.44 & \textbf{17.83} & \textbf{43.86} & 73.61 & 57.79 & 59.20
& \underline{69.74} & \textbf{40.10} & 41.91 
& \underline{87.50}
& \textbf{59.21} \\

\bottomrule
\end{tabular}
}
\caption{Downstream language modeling evaluation on 1.3B-parameter models trained on 100B FineWeb-Edu tokens. Best scores are in bold and second-best are underlined.}
\label{tab:lm_eval_results_withd4}
\end{table}
\begin{wrapfigure}{r}{0.22\textwidth}
  \centering
  \vspace{-1.5em}
  \includegraphics[width=\linewidth]{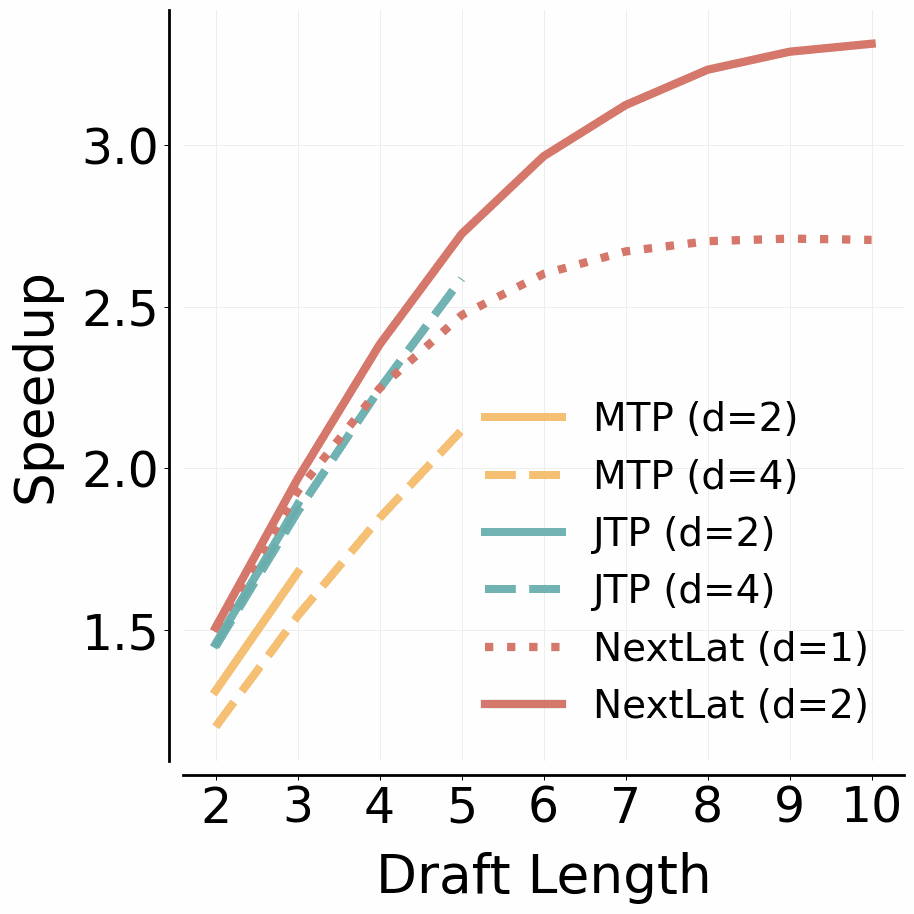}
  \caption{Inference speedup ratio on FineWeb-Edu.}
  \label{fig:fineweb_spec_results_withd4}
  \vspace{-1em}
\end{wrapfigure}
In this section, we extend our comparisons to include JTP and MTP trained with larger horizons ($d=4$). Note that the training cost of JTP and MTP increases substantially with larger $d$ (see \cref{tab:fineweb_compute}), making JTP ($d=4$) and MTP ($d=4$) significantly more expensive to train than NextLat ($d=1,2$). 

\cref{tab:lm_eval_results_withd4} shows that increasing the multi-token prediction horizon to $d=4$ still does not yield meaningful improvements on multiple-choice benchmarks; NextLat ($d=2$) continues to achieve the best average accuracy overall. \cref{tab:spec_eval_results_withd4} shows that even when JTP and MTP are trained to be able to draft more tokens (i.e., 4 tokens ahead), they still fail to surpass the variable-length speculative decoding performance of NextLat across most domains, with the exception of the Code domain. Finally, \cref{fig:fineweb_spec_results_withd4} shows that, on the FineWeb-Edu validation set, increasing the training horizon to $d=4$ yields only modest speedup gains for JTP and MTP. Their speedup curves saturate earlier and remain below that of NextLat ($d=2$), which continues to improve with longer draft lengths and achieves the highest overall speedup.

\begin{table}[H]
\centering
\renewcommand{\arraystretch}{0.9}
\resizebox{0.9\columnwidth}{!}{%
\begin{tabular}{lcc|cc|cc|cc}
\toprule
 & \multicolumn{2}{c}{\textbf{Wikipedia}} 
 & \multicolumn{2}{c}{\textbf{Books}} 
 & \multicolumn{2}{c}{\textbf{Code}}
 & \multicolumn{2}{c}{\textbf{Math}} \\
\cmidrule(lr){2-3} \cmidrule(lr){4-5} \cmidrule(lr){6-7} \cmidrule(lr){8-9}
\textbf{Model} 
& Speedup & Accepted Tokens 
& Speedup & Accepted Tokens 
& Speedup & Accepted Tokens
& Speedup & Accepted Tokens \\
\midrule
JTP (d=1) & 1.46 & 0.96 & 1.47 & 0.97 & 1.47 & 0.98 & 1.46 & 0.97  \\
JTP (d=2) & 1.88 & 1.84 & 1.90 & 1.89 & 1.88 & 1.85 & 1.89 & 1.86 \\
JTP (d=4) & 2.58 & 3.32 & 2.62 & 3.43 & \textbf{2.61} & \underline{3.43} & \underline{2.42} & 3.02 \\
MTP (d=1) & 1.38 & 0.91 & 1.39 & 0.95 & 1.40 & 0.97 & 1.39 & 0.95 \\
MTP (d=2) & 1.68 & 1.72 & 1.72 & 1.83 & 1.75 & 1.91 & 1.72 & 1.84 \\
MTP (d=4) & 2.10 & 3.04 & 2.25 & 3.44 & 2.32 & \textbf{3.68} & 2.25 & \underline{3.46} \\
\midrule
\textbf{NextLat (d=1)} & \underline{2.68} & \underline{3.52} & \underline{2.72} & \underline{3.64} & 2.29 & 2.66 & 2.30 & 2.72 \\
\textbf{NextLat (d=2)} & \textbf{3.21} & \textbf{4.59} & \textbf{3.32} & \textbf{4.86} & \underline{2.38} & 2.83 & \textbf{2.87} & \textbf{3.94} \\
\bottomrule
\end{tabular}%
}
\caption{Relative speedup and average accepted tokens per drafting step over diverse domains. Note that ``Accepted Tokens'' excludes the next-token prediction which is always accepted.}
\label{tab:spec_eval_results_withd4}
\end{table}